\documentclass[letterpaper,twocolumn,10pt]{article}
\usepackage{usenix}

\usepackage{amsmath}
\usepackage{amssymb}
\usepackage{mathtools}
\usepackage{amsthm}
\usepackage{amsfonts}       
\usepackage{thmtools,thm-restate}
\usepackage{dsfont}
\usepackage{authblk}
\usepackage{pdfpages}
\usepackage[available,functional,reproduced]{usenixbadges}

\usepackage{hyperref}
\usepackage{tikz}
\usepackage{amsmath}
\usepackage{filecontents}
\usepackage{footnote}
\usepackage{listings}

\usepackage{subfigure}
\usepackage{hyperref}       
\usepackage{url}            
\usepackage{booktabs}       
\usepackage{microtype}      
\usepackage{graphicx}%
\usepackage{subcaption}
\usepackage{float}
\usepackage{soul}
\usepackage{caption}
\usepackage{multirow}
\usepackage{bbm}
\usepackage{bm}
\usepackage{mathrsfs}
\usepackage{xcolor}
\usepackage{xspace}
\usepackage{tikz}
\usepackage[framemethod=tikz]{mdframed}

\usepackage{thmtools,thm-restate}
\usepackage{dsfont}
\usepackage[ruled,vlined,linesnumbered]{algorithm2e}
\usepackage{algorithm2e}
\usepackage{soul}
\usepackage{algpseudocode}  
\usepackage{pifont}
\usepackage[flushleft]{threeparttable}
\usepackage{xcolor}
\usepackage{makecell}
\usepackage{lipsum}
\usepackage{balance}
\usepackage{caption} 
\usepackage{enumitem}
\setlist{noitemsep}

\usepackage{url}

\definecolor{firebrick}{rgb}{0.7, 0.13, 0.13}
\definecolor{darkblue}{rgb}{0,0,0.55}
\definecolor{grey}{rgb}{0.8,0.8,0.8}



\def\eg{\emph{e.g.,}\xspace}

\def\ie{\emph{i.e.,}\xspace}

\def\eqref#1{equation~\ref{#1}}

\def\floor#1{\lfloor #1 \rfloor}
\def\1{\bm{1}}

\def\rvr{{\mathbf{r}}}

\def\rvz{{\mathbf{z}}}

\def\rmR{{\mathbf{R}}}

\def\rmZ{{\mathbf{Z}}}

\DeclareMathAlphabet{\mathsfit}{\encodingdefault}{\sfdefault}{m}{sl}
\SetMathAlphabet{\mathsfit}{bold}{\encodingdefault}{\sfdefault}{bx}{n}
\newcommand{\tens}[1]{\bm{\mathsfit{#1}}}

\def\tE{{\tens{E}}}

\def\gD{{\mathcal{D}}}

\def\gR{{\mathcal{R}}}

\def\gT{{\mathcal{T}}}

\def\gZ{{\mathcal{Z}}}

\def\sA{{\mathbb{A}}}

\def\sS{{\mathbb{S}}}

\newcommand{\E}{\mathbb{E}}

\newcommand{\R}{\mathbb{R}}

\newcommand{\N}{\mathcal{N}}

\DeclareMathOperator*{\argmax}{arg\,max}
\DeclareMathOperator*{\argmin}{arg\,min}

\def\eg{\emph{e.g.,}\xspace}

\def\ie{\emph{i.e.,}\xspace}

\graphicspath{{./Figures/}}
\DeclareUnicodeCharacter{0301}{\hspace{-1ex}\'{ }}
\begin{document}

\date{}

\title{\Large \bf CAMP in the Odyssey: \\
Provably Robust Reinforcement Learning with Certified Radius Maximization}

\author[$\heartsuit$,$\spadesuit$]{Derui Wang}
\author[$\heartsuit$,$\spadesuit$]{Kristen Moore}
\author[$\heartsuit$,$\spadesuit$]{Diksha Goel}
\author[$\heartsuit$,$\spadesuit$]{Minjune Kim}
\author[$\clubsuit$]{Gang Li}
\author[$\clubsuit$]{Yang Li}
\author[$\clubsuit$]{Robin Doss}
\author[$\heartsuit$,$\spadesuit$]{Minhui Xue}
\author[$\lozenge$]{Bo Li}
\author[$\heartsuit$,$\spadesuit$]{Seyit Camtepe}
\author[$\heartsuit$]{Liming Zhu}
\affil[$\heartsuit$]{\textit{CSIRO's Data61, Australia}}
\affil[$\spadesuit$]{\textit{Cyber Security Cooperative Research Centre, Australia}}
\affil[$\clubsuit$]{\textit{Deakin University, Australia}}
\affil[$\lozenge$]{\textit{University of Chicago, USA}}

\maketitle

\begin{abstract}
Deep reinforcement learning (DRL) has gained widespread adoption in control and decision-making tasks due to its strong performance in dynamic environments. 
However, DRL agents are vulnerable to noisy observations and adversarial attacks, and concerns about the adversarial robustness of DRL systems have emerged. 
Recent efforts have focused on addressing these robustness issues by establishing rigorous theoretical guarantees for the returns achieved by DRL agents in adversarial settings. 
Among these approaches, policy smoothing has proven to be an effective and scalable method for certifying the robustness of DRL agents.
Nevertheless, existing certifiably robust DRL relies on policies trained with simple Gaussian augmentations, resulting in a suboptimal trade-off between certified robustness and certified return. 
To address this issue, we introduce a novel paradigm dubbed \texttt{C}ertified-r\texttt{A}dius-\texttt{M}aximizing \texttt{P}olicy (\texttt{CAMP}) training. 
\texttt{CAMP} is designed to enhance DRL policies, achieving better utility without compromising provable robustness. 
By leveraging the insight that the global certified radius can be derived from local certified radii based on training-time statistics, \texttt{CAMP} formulates a surrogate loss related to the local certified radius and optimizes the policy guided by this surrogate loss. 
We also introduce \textit{policy imitation} as a novel technique to stabilize \texttt{CAMP} training.
Experimental results demonstrate that \texttt{CAMP} significantly improves the robustness-return trade-off across various tasks. 
Based on the results, \texttt{CAMP} can achieve up to twice the certified expected return compared to that of baselines.
Our code is available at \href{https://github.com/NeuralSec/camp-robust-rl}{https://github.com/NeuralSec/camp-robust-rl}.
\end{abstract}

\section{Introduction}
Deep reinforcement learning (DRL) has widespread applications in various areas, including robotics~\cite{levine2016end,andrychowicz2020learning,fernandez2023deep}, games~\cite{silver2017mastering,mnih2015human}, dialogue systems~\cite{li2017end}, and finance~\cite{yang2020deep}.
However, DRL agents are vulnerable to adversaries who actively perturb the agents' observations. These perturbations cause agents to take sub-optimal actions, potentially leading to severe consequences in critical missions~\cite{huang2017adversarial,behzadan2017vulnerability,pattanaik2018robust,gleave2020adversarial,sun2020stealthy,lin2020robustness}. 
In robotics applications, attackers manipulating DRL agents's observations can cause permanent damage to the robots or even create life-threatening situations. 
On the other hand, attacks against DRL in dialogue or trading systems may also lead to biased or harmful decisions by the system.
Adversarial attacks against DRL agents typically involve perturbing state observations of the agents or directly modifying their actions. In experimental settings and simulations, the impact of these attacks is often quantified by the sub-optimal rewards obtained by the compromised agents.

Empirical defenses have been developed for DRL agents to mitigate the risk posed by deliberate adversaries.
One line of work focuses on training robust DRL policies, aiming to make agents less sensitive to changes in their observations~\cite{zhang2020robust,shen2020deep, yang2022rorl}.
Methods also exist that model robust Markov Decision Processes (MDPs) to enhance the performance of DRL in noisy or adversarial environments~\cite{bander1999markov,zhang2020robust,liu2023robustness}.
However, the empirical defenses fail to provide rigid guarantees of the agents' adversarial robustness, leaving them vulnerable to cross-attack generalization and adaptive attacks. 

In efforts to verify DRL robustness, randomized smoothing (RS) has emerged as a powerful and scalable tool for both enhancing and certifying the robustness of black-box functions in real-world applications~\cite{cohen2019certified}.
RS has been applied to DRL policies through policy smoothing, achieving provable robustness against adversarial attacks~\cite{kumar2021policy}.
Compared to empirical DRL defenses (\eg adversarial training~\cite{kamalaruban2020robust,vinitsky2020robust}), RS provides guaranteed lower or upper bounds on the expected return (\ie expected cumulative discounted rewards going forward) with respect to a certain amount of changes in the observation.
Given that smoothing noise from previous steps influences both observations and subsequent smoothing noise, RS in DRL typically constructs a specially structured adversary for certification. This approach enables certification under the Neyman-Pearson lemma or other algorithmic variants~\cite{kumar2021policy, wu2022crop, mu2024reward}.

\noindent \textbf{Motivation and challenges.~} 
A direct motivation for our work is the absence of methods specifically designed to train DRL agents with RS-based provable robustness. 
Certified robustness through RS often presents challenges that existing methods cannot effectively address.
Specifically, there is a trade-off between the certified expected return and the certified radius of adversarial perturbations, leading to two key consequences:
1) First, certifying an expected return against adversaries using large perturbations becomes infeasible, since the certified expected return rapidly drops to zero as the certified radius increases.
2) Additionally, the certified expected return of a randomized agent is significantly lower than the standard return the same agent can achieve in a non-randomized environment, underscoring the need to improve the overall certified expected return.
These consequences hinder the deployment of policy smoothing in real-world DRL applications, where both certified robustness and certified utility are crucial.

Existing provable defenses train the policy for smoothing via applying Gaussian noise to the observations~\cite{kumar2021policy}.
The noise applied during training should match the smoothing noise used in the test phase to prevent degradation in the test-time reward. 
However, such a policy may not achieve optimal certification results, as the augmentation does not directly maximize the certified radius during training, leaving the correlation between the augmentation and the trade-off unclear. 
On the other hand, there also exist robust DRL training methods that enhance the robustness and smoothness of DRL agents operating in noisy or adversarial environments~\cite{shen2020deep, yang2022rorl, zhang2020robust, liu2023robustness}.
However, instead of improving the certified expectations of returns received through randomized policies, these methods primarily focus on the epistemic robustness of deterministic policies.
Moreover, the desired goals of robustness and utility cannot be easily achieved using existing methods for training classifiers with certified robustness~\cite{salman2019provably,zhai2020macer,jeong2021smoothmix,jeong2020consistency}.
The reasons for this are 1) DRL tasks are MDPs, which introduce optimization objectives that diverge from that of classification tasks.
2) The certified radius of perturbations is determined by numerical methods like binary search, which makes it difficult to directly optimize the radius in order to enhance the robustness. 
These two key characteristics of DRL necessitate novel paradigms in the construction and training of the policy.

\noindent \textbf{Our stance and contributions.~}
Motivated by these factors, we propose \texttt{C}ertified-r\texttt{A}dius-\texttt{M}aximizing \texttt{P}olicy (\texttt{CAMP}) training, a method for developing DRL agents that achieve better certified expected returns without compromising certified robustness.
Theoretically, we derive a substitute certified radius that can be maximized in the training, in parallel to the utility maximization process.
This enables \texttt{CAMP} to mitigate the trade-off between the certified expected return and the certified radius by offering a superior base function.
Our focus is on agents trained using deep Q-learning~\cite{mnih2015human} in environments with discrete action spaces.
Through a change of variables, we formally define the certified radius as a function of a target threshold for the certified expected return. 
This certified radius is then transformed into a differentiable format for further analysis.
Building on the insight that the decline in optimal certified expected return is due to accumulated per-step errors, we derive a local certified radius that serves as a surrogate loss for optimizing the certified radius.

Furthermore, to enhance the versatility and training stability of \texttt{CAMP}, we introduce a novel training paradigm called policy imitation. 
Since the application of \texttt{CAMP} can influence the Bellman error, it may result in the overestimation of Q-values during training. 
This issue can impede the convergence of the training process and often leads to suboptimal policies~\cite{van2016deep,peer2021ensemble}.
Previous solutions, such as Double DQN~\cite{van2016deep}, ensemble bootstrapping~\cite{peer2021ensemble}, and conservative smoothing~\cite{shen2020deep,yang2022rorl}, are not designed to directly address this problem within the \texttt{CAMP} framework. 
Consequently, a new training paradigm is needed to effectively integrate \texttt{CAMP} into DRL training pipelines.
Importantly, we aim for \texttt{CAMP} to be applicable across various learning environments. 
Policy imitation addresses the aforementioned challenges by using a reference network to model the oracle Q-values, which can then be used to constrain the primary policy during training.
The reference network is trained using conventional temporal-difference methods, while the primary network imitates the reference network, minimizing the difference between their predicted actions for the same observations. 
Additionally, the \texttt{CAMP} loss is applied to the primary policy to enhance its certified robustness radius. 
To summarize, our contributions are as follows:
\begin{itemize}[leftmargin=*]
    \item We analyze the certified radius for globally smoothed policies and reformulate it as a radius maximization objective, which can be directly optimized during the training process.
    \item We propose a novel policy imitation training paradigm that integrates certified radius maximization into deep Q-learning. 
    \item \texttt{CAMP} achieves higher certified expected returns at larger certified radii and demonstrates superior empirical robustness in individual game episodes.
\end{itemize}

In the following paper, we will first introduce the necessary background knowledge, along with the definitions of the problem and threat model. 
We will then proceed to present \texttt{CAMP} and discuss the associated experiments.

\section{Preliminaries}
To begin, we provide a brief overview of deep Q-learning, adversarial attacks on DRL, and the certified robustness of DRL. 
This paper focuses on discrete action spaces and emphasizes robustness certification using policy smoothing and the Neyman-Pearson Lemma, as these methods offer a tight bound on the certified expected return.

\subsection{Deep Q-Learning}\label{subsec:drl_background}
A reinforcement learning task is usually defined as a discrete MDP.
Specifically, an agent interacts with an environment $\tE = (\sS, \sA, \gT, \gR, \gamma, s_{0})$, where $\sS\in\R^d$ and $\sA\in\R^k$ are the space of states and the space of actions, respectively. 
$\gT(\cdot)$ is a state transition function and $\gR(\cdot)$ is a reward function.
$s_0$ comes from an initial distribution of states. 
$\gamma\in[0,1]$ is a discount factor for the reward at different steps.
Furthermore, the state at the $t+1$-th step is determined by $s_{t+1} = \gT(s_{t}, a_{t})$.
In this paper, following previous certification settings, the states are fully observable by a deterministic observation function.
Therefore, we use $s$ interchangeably to denote both the state and the observation.
Given the state $s_{t}$ and the corresponding action $a_{t}$ at the $t$-th step, the reward $r_{t+1}$ acquired by $a_{t}$ is computed by the reward function as $r_{t+1} = \gR(s_{t}, a_{t})$.

In practice, the transition function is determined by nature and is therefore unknown, which encourages the development of model-free RL algorithms. 
Among these algorithms, Q-learning~\cite{watkins1992q} aims to model the action-value function (\ie Q-function) $Q_{\pi}(s, a)$ which estimates the expected future return for taking action $a$ in state $s$ and following a policy.
DQN models $Q_{\pi}(s, a)$ with a neural network (Q-network) and uses a greedy policy to select actions.
For simplicity, we use $\pi$ to denote both the \textit{Q-network} and its \textit{parameters} in this paper.
Thus, at the $t$ step, the action selection policy can be formulated as taking an observed state $s_t$, predicting a probability simplex $\{\pi(s_{t})_i\}_{i=1}^{|\sA|}$ over $\sA$, and then selecting $a = \argmax_{i} \pi(s_{t})_i$.

To update the network parameters $\pi$ towards an optimal Q-function, a Temporal-Difference (TD) loss is defined as the difference between the current Q-value and an estimated optimal value.
Specifically, the optimal Q-value is the solution to the Bellman equation, which can be calculated as:
\begin{equation}\small
    Q_{\pi^*}(s,a) = \gR(s, a) + \gamma \underset{s'\sim\gT(s,a)}{\E} \left[ \max_{a^*} Q_{\pi^*}(s', a^*) \right],
\end{equation}
where $\gamma \in [0,1]$ is a discount factor penalizing the rewards in the future.
$Q_{\pi^*}(s', a^*)$ is the optimal Q-value in which $s'$ is the next state according to the transition function $\gT(s, a)$ and $a^*$ is the best action estimated greedily for the next step by the optimal network denoted as $\pi^*$.
The overall training objective of the agent can be formulated as minimizing the mean squared Bellman error between the current Q-function and the optimal Q-function:
\begin{equation}\small\label{eq:MSBE_dqn}
    \begin{aligned}
        \min_{\pi} \underset{(s, a, s')\sim \rmZ}\E \left[ Q_{\pi}(s,a) - \left( \gR(s, a) + \gamma \max_{a'} Q_{\pi^*}(s', a') \right) \right]^2,
    \end{aligned}
\end{equation}
where $\rmZ$ is the space of state-action trajectories.
In practice, an experience replay buffer or an offline dataset stores a set of previously collected trajectories as samples from $\rmZ$, such that the above training objective can be solved by empirical risk minimization.

\subsection{Expected Return Certification against Perturbed Observations}\label{subsec:cdf_smoothing}
\noindent \textbf{Adversarial observation perturbations against DRL.~}
An adversary targeting a DRL agent aims to perturb the agent's observations during test time, thereby minimizing the return within a finite horizon $T$.
A simple adversarial attack adapts the Fast Gradient Sign Method (FGSM) to perturb observations with a sequence of perturbations $\Delta=(\delta_0,\delta_1,...,\delta_{T-1})$, $\Delta\in\R^{d\times T}$ and alter the actions of the DRL agent, causing the agent to reject the correct action and instead select the one with the minimal Q-value~\cite{huang2017adversarial}.
On the other hand, another type of attack finds a perturbation sequence $\Delta$, as the solution of
\begin{equation}\small\label{eq:adv_objective}
    \begin{aligned}
        \argmin_{\Delta}\ & \underset{\delta_t\in\Delta, s_t}{\E} & \, \left[ \sum_{t=0}^{T-1} \gamma^t \gR\left(s_t, \argmax_{a}\pi(s_t+\delta_t)_a\right) \right],\\
    \end{aligned}
\end{equation}
In practice, the attacker can find per-state perturbation as $\delta_t = \argmin_{\delta_t} \gD_{J}\left( \pi(s_t) \| \pi(s_t+\delta_t)\right)$, where $\gD_{J}$ is the Jeffrey's Divergence~\cite{yang2022rorl}.
In both attacks, the adversarial perturbations in the sequence $\Delta$ are collectively constrained within an region $B(\Delta)=\{(\delta_0,\delta_1,...,\delta_{T-1}):\, \sqrt{\sum_{t=0}^{T-1}\|\delta_t\|_{2}^2}\leq \tau \}$ under a budget of $\tau$.

\noindent \textbf{Robustness certification via policy smoothing.~} 
In response, policy smoothing emerges as a provable defense against adversarial attacks. Policy smoothing involves adding noise sampled from a Gaussian distribution $\N(0, \sigma^2I)$ with variance $\sigma^2$ to state observations at each step, thereby randomizing the state-action trajectory of the DRL agent. 
By repeating this process across multiple runs, a set of randomized trajectories along with their corresponding rewards can be collected, allowing the expected return of the agent to be lower bounded.

Specifically, let the sequence of the smoothing noises across $T$ steps (\eg within an episode) be $\varepsilon = (\epsilon_0, \epsilon_1,...\epsilon_{T-1})$, the defender can sample random trajectories $\rvz = (s_0+\epsilon_0, a_0, ..., s_{T-1}+\epsilon_{T-1}, a_{T-1})$, $\rvz \in \rmZ$, where $\rmZ$ is the space of randomized state-action trajectories.
Furthermore, a randomized reward $r_{t+1} = \gR(s_t, \argmax_a \pi(s_t+\epsilon_t)_a)$ is received given the observed state and action output by the policy.
The return obtained by a random trajectory is $\rvr = \sum_{t=0}^{T-1} \gamma^t r_{t+1}$ and $\gamma$ is the discount factor.
Therefore, such a return can be viewed as a randomized return function $\rvr = F_{\pi}(\rvz)$ of the trajectory $\rvz$, parameterized by the policy parameters $\pi$.
Let us consider the probability $P_{\pi}^{\rvz}(C) = \Pr[ F_{\pi}(\rvz) \geq C ]$, and we can accordingly define $\Psi(C) = 1 - P_{\pi}^{\rvz}(C)$.
It is easy to notice that $\Psi(C)$ is a cumulative distribution function (CDF), \ie $\Psi(C) = \Pr[ F_{\pi}(\rvz) \leq C]$.
By sampling $F_{\pi}(\rvz)$ values for $m$ times via allowing the agent to play on random trajectories for $m$ runs, a set $\rmR = \{\rvr_1, ..., \rvr_m\}$ of return values sorted in ascending order can be obtained to compute an empirical cumulative distribution function (ECDF) $\tilde{\Psi}(C) = \sum_{i=1}^{m}\mathds{1}_{\{\rvr_i< C\}}/m, \forall \rvr_i\in \rmR$. 
Next, the interval $[\underline{\Psi}(C), \overline{\Psi}(C)]$ of the CDF ${\Psi}(C)$ can be computed from the ECDF by Dvoretzky–Kiefer–Wolfowitz Inequality~\cite{dvoretzky1956asymptotic}:
\begin{equation}\small
    \overline{\Psi}(C) = \tilde{\Psi}(C) + \sqrt{\frac{\ln(2/\alpha)}{2m}},
\end{equation}
with probability $1-\alpha$ and the draw count $m$.
Since $ {\Psi}(C) = 1 - P_{\pi}^{\rvz}(C)$, the upper bound $\overline{\Psi}(C)$ can thus be used to obtain a probabilistic lower bound of $P_{\pi}^{\rvz}(C)$,  which gives $\underline{P_{\pi}^{\rvz}}(C) = 1-\overline{\Psi}(C)$.

 Consider an adversary who adds a perturbation sequence $\Delta=(\delta_1,\delta_2,...,\delta_{T-1})$ under a $\ell_2$-norm budget of $\tau$ (\ie $\|\Delta\|_2=\sqrt{\sum_{t=0}^{T-1}\|\delta_t\|_2^2},\, \|\Delta\|_2 \leq \tau$) to the state observations and let the perturbed random trajectory be $\rvz' := \rvz + \Delta = (s'_0+\epsilon_0, a'_0, ..., s'_{T-1}+\epsilon_{T-1}, a'_{T-1})$ in which $s'_t = s_t + \delta_t$ and $a'_t = \argmax_{a}\pi(s_t + \epsilon_t + \delta_t)_a$.
 The following perturbation bound on the probability $P_{\pi}^{\rvz'}(C)$ can be established by an adaptive Neyman-Pearson Lemma based on a structured adversary~\cite{kumar2021policy}:
 \begin{equation}\small\label{eq:cdf_bound}
\begin{aligned}
    P_{\pi}^{\rvz'}(C) \geq \Phi \left(\Phi^{-1} \left( \underline{P_{\pi}^{\rvz}}(C) \right)-\frac{\tau}{\sigma} \right),\, \forall\, \|\Delta\|_2 \leq \tau.
\end{aligned}
\end{equation}
Furthermore, since the expectation of $F_{\pi}(\rvz)$ can be represented as the integral of $P_{\pi}^{\rvz}(C)$ over $C$ values~\cite{kumar2020certifying}, a lower bound on the expected return under adversarial perturbations $\Delta$ can be obtained, given $C\in\rmR$, as:
\begin{equation}\small\label{eq:robustness_bound}
    \begin{aligned}
        \E[F_{\pi}(\rvz')] \geq & \rvr_1 \cdot \Phi \left(\Phi^{-1} \left( \underline{P_{\pi}^{\rvz}}(\rvr_1) \right)-\frac{\tau}{\sigma} \right) \\
        + & \sum_{i=2}^{m}(\rvr_i - \rvr_{i-1}) \cdot \Phi \left(\Phi^{-1} \left( \underline{P_{\pi}^{\rvz}}(\rvr_i) \right)-\frac{\tau}{\sigma} \right), \\
        \forall\, & \|\Delta\|_2 \leq \tau.
    \end{aligned}
\end{equation}
On the right-hand side, $\Phi(\cdot)$ is the cumulative density function (CDF) of standard Gaussian and $\Phi^{-1}(\cdot)$ is its inverse.

\begin{figure*}[t]
    \centering
    \includegraphics[width=1.\linewidth]{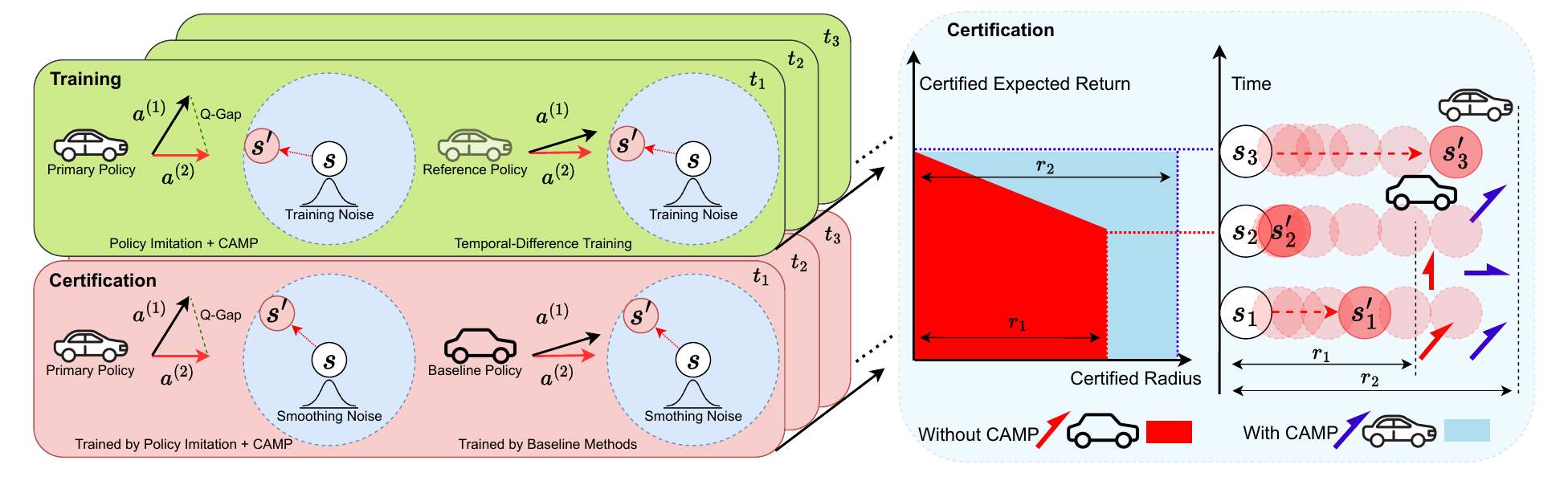}  
    \caption{
    An overview of \texttt{CAMP}. During training, the agents interact with the environments with observation noise and receive rewards.
    A reference policy has a reference Q-network learning through a vanilla temporal-difference loss while the Q-network of a primary policy is optimized by minimizing the \texttt{CAMP} loss to increase the gap between the top-1 and runner-up Q-values.
    The primary network also mimics the action predicted by the reference network during training with the imitation loss.
    The trained primary policy can then be certified by policy smoothing with better certified expected return at each certified radius.
    }
    \label{fig:overview}
\end{figure*}

\section{Problem Formulation}
In this section, we present precise formulations of the research problem. 
We also define the threat model to reflect the settings for training and certifying DRL agents, as well as the key properties of the adversary.

\subsection{Problem Statement}\label{sec:problem_statement}
We aim to maximize both the expected return and the robustness of the DRL agent. 
This desideratum can be decomposed into two sub-objectives: minimizing utility cost and minimizing robustness cost. 
Moreover, both cost minimizations should be seamlessly integrated into the training process of the DRL agent.

Following the above notations and Equations, we can observe that the lower bound of the expected return in Equation~\ref{eq:robustness_bound} has a positive correlation with the probabilities in the set $\{\underline{P_{\pi}^{\rvz}}(\rvr_i)\}_{i=1}^{m}$.
This observation implies that there is a monotonic connection between the utility and the probability of obtaining a return above each threshold $C_i$.
We can thus instantly formulate a utility cost as:
\begin{equation}\small\label{eq:utility_loss_intractable}
    \begin{aligned}
        & \ell_{ut}(\pi) = \sum_{i=1}^{m} \left( 1 - \underline{P_{\pi}^{\rvz}}(\rvr_i) \right).
    \end{aligned}
\end{equation}
However, this formulation makes minimizing $\ell_{ut}(\pi)$ an intractable problem due to non-differentiable $\underline{P_{\pi}^{\rvz}}(\rvr_i)$ and the sampling of $\rvz$ and $\rmR$.
To optimize $\ell_{ut}(\cdot)$, we will construct a differentiable surrogate loss in Section~\ref{sec:method} as an upper bound of Equation~\ref{eq:utility_loss_intractable}.

On the side of robustness, we also aim to derive a loss function indicating the extent of robustness.
It is natural to consider a closed-form representation of the certified radius as a function of $\pi$ to measure the robustness.
However, in policy smoothing, the certified radius (\ie the maximal $\tau$ making Equation~\ref{eq:robustness_bound} hold) is a variable rather than a function and is usually determined by binary search.
To address this challenge, it is noticeable that the certified lower bound in Equation~\ref{eq:robustness_bound} monotonically decreases with increasing $\tau$.
Therefore, the binary search can be viewed as finding a $\tau$ value such that the lower bound is no less than a predetermined value $\xi$.
To clearly formulate the process of optimizing $\tau$, our analysis stems from the following theorem:
\begin{restatable}
[\textit{Change of variable}]{theorem}{hardradius}\label{lemma:hardradius}
Given a target expected return threshold $\xi$ for the randomized policy, let the perturbed trajectory be $\rvz' = \rvz+\Delta$ and define $P_{\pi}^{\rvz}(C) := \Pr[ F_{\pi}(\rvz) \geq C ]$. 
Let $\rmR = \{\rvr_1,...,\rvr_m\}$ represent a set of sampled and sorted values of $F_{\pi}(\rvz)$ such that $\rvr_1 \leq \rvr_2 \leq ... \leq \rvr_m$.
If $\xi/\rvr_1 \leq \underline{P_{\pi}^{\rvz}}(\rvr_1))$ and
\small
\begin{equation}\label{eq:substitute_radius}
    \|\Delta\|_2 \leq \sigma \left[ \Phi^{-1}(\underline{P_{\pi}^{\rvz}}(\rvr_1)) -\Phi^{-1}(\xi/\rvr_1) \right],
\end{equation}
\normalsize
then $\E[F_{\pi}(\rvz')] \geq \xi$.
\end{restatable}
\noindent
The proof of the theorem is in Appendix~\ref{append:proof}.
Theorem~\ref{lemma:hardradius} implies that $\tau = \sigma [ \Phi^{-1}(\underline{P_{\pi}^{\rvz}}(\rvr_1)) -\Phi^{-1}(\xi/\rvr_1) ]$,  thereby converting $\tau$ from a variable to a function of $\xi$.
In this way, it defines a substitute certified radius and characterizes the relationship between this radius, the Q-network $\pi$, the statistics $\underline{P_{\pi}^{\rvz}}(\rvr_1)$, and $\xi$.
However, the substitute certified radius remains non-differentiable at this stage.
Therefore, in Section~\ref{sec:method}, we propose a method to maximize the certified robustness radius by optimizing each local certified radius over a finite horizon $T$.

\subsection{Threat Model}
\texttt{CAMP} operates during the training stage of DRL agents. 
The defender, responsible for both model training and robustness certification, follows the typical practices of DRL trainers. 
The defender must have access to the observations of the DRL agent and can modify these observations with random noise or perturbations. 
This is usually feasible, as the defender typically sets up the environment in which the DRL agent is trained. 
Additionally, as the model trainer, the defender has full access to the Q-network and the policy, including intermediate weights and predicted Q values based on observations. 
For certification purposes, the defender can test the trained agent in various environments over multiple rounds and record the actual returns from each round.

The adversaries against which we certify our defense are those who perturb the observed states of the DRL agent to mislead its actions. 
For instance, in classic control applications, such as robotic arm manipulation, the observed states may consist of sensor data related to the arm's kinematics. 
An attacker could manipulate this sensor data to induce abnormal movements in the robotic arm, while the underlying state used for transition and reward computation remains unchanged.
This observation-based adversary is more practical than attackers perturbing actions or policy parameters in real-world scenarios, as it is generally easier to tamper with sensor data than to manipulate agent actions or alter parameters in real time.
Furthermore, this type of adversary presents significant challenges for certified robustness due to the non-convex and black-box nature of policy functions, which motivates the adoption of RS for certification.

Additionally, adversaries can distribute perturbations across a finite number of transitions in the environment, provided that the total $\ell_2$ norm of the perturbations remains within a specified budget. 
This setup allows adversaries to flexibly distribute perturbations, leading to stronger attacks, measurable by reduced DRL agent returns.
The adversary can access the victim agent's Q-network in either a white-box or black-box manner, allowing them to optimize perturbations more effectively.
In white-box attacks, adversaries have access to the network's parameters, states, and actions, while black-box attacks rely on querying the Q-network with observed states to infer the corresponding actions.

\subsection{Design Intuition}
In this paper, we propose \texttt{CAMP} as a method for training provably robust DRL agents in discrete action spaces through certified radius maximization. 
\texttt{CAMP} tackles this challenge by formulating the robustness radius and transforming it into a robustness cost that can be minimized in the deep Q-learning process.
To start with, we first convert the certified radius into a function of the Q-network $\pi$, the statistics $\underline{P_{\pi}^{\rvz}}(\rvr_1)$, and a threshold $\xi$.
Building on the Lipschitz continuity of the expected return function with respect to perturbations in the observations, we proceed to derive a soft radius. 
In addition, the observation noise perturbs the agent's action selection, thereby affecting the action-value function at each decision step. 
A robust policy is one that consistently selects the action yielding the highest expected Q-value in the presence of observation noise at every step. 
From this principle, we derive local radii that exhibit a positive correlation with the gap between the top-1 and runner-up Q-values.

Importantly, we find that the robustness cost cannot be minimized with the TD loss at the same time, as it destabilizes the training process.
To overcome this barrier, we propose adopting policy imitation as a training paradigm.
Policy imitation creates a policy based on a reference network to guide the training of the Q-network of the primary policy.
The primary network mimics the action selected by the reference network but does not directly approach the Q-value of the state-action pair.
On the other hand, the primary policy minimizes its robustness cost during the imitation, such that agents with this trained primary policy can be certified at a greater robustness radius given a fixed certified expected return.
The detailed design of \texttt{CAMP} will be introduced in the next section.

\section{Provable Robustness via \texttt{CAMP}}\label{sec:method}
This section introduces the design of \texttt{CAMP}. The core objective of \texttt{CAMP} is to train the agent to maximize the certified utility (e.g., achieve a higher certified expected return) while optimizing an approximated, differentiable certified radius.
We first describe the steps leading to the differentiable local certified radius and demonstrate how maximizing this radius improves certification performance. 
Next, to prevent overestimation of Q-values during the training of \texttt{CAMP}, we introduce \textit{policy imitation} as a novel training paradigm. 
Policy imitation employs a secondary Q-network as an oracle to guide the training of the primary policy.
Finally, we certify the expected returns of the \texttt{CAMP} agent through policy smoothing. 
An overview of the \texttt{CAMP} pipeline is provided in Figure~\ref{fig:overview}.

\subsection{Expected Return Maximization}
Our first desideratum is to maximize the certified utility of the certifiably robust DRL agent. 
In reference to the objective outlined in Section~\ref{sec:problem_statement}, our goal is to find a differentiable upper bound for Equation~\ref{eq:utility_loss_intractable}.
The initial challenge is to address the sampling process of $\rvz$ and $\rmR$. 
Intuitively, increasing the probability that the return from random trajectories exceeds a certain threshold can be equated to maximizing the expectation of returns over randomness in some special cases.
First, we restate the following well-known theorem:
\begin{restatable}
[\textit{Correlation between CDF and expectation}]{theorem}{cdftoexpectation}\label{theorem:cdftoexpectation}
Given two random variables $X \geq 0$ and $Y \geq 0$, let their expectations be $\E[X]$ and $\E[Y]$, respectively. 
The following inequality holds if and only if $\E[X] \leq \E[Y]$:
\begin{equation}\small
    \int_{0}^{+\infty} [ \Psi_{X}(C) - \Psi_{Y}(C) ] \mathrm{d}C \geq 0,
\end{equation}
where $\Psi_{X}(C)$ and $\Psi_{Y}(C)$ are the CDFs of $X$ and $Y$, respectively.
\end{restatable}  

The theorem suggests that increasing the expectation of a random variable may increase the chance of the random variable exceeding a given threshold, thereby decreasing the integral of the CDF.
In our case, given the random variable $F_{\pi}(\rvz)$, we can increase its expected value such that Equation~\ref{eq:utility_loss_intractable} can be reduced, as it is positively correlated with the CDF.
This motivates us to formulate the utility maximization objective as minimizing the expected TD loss, $\ell_{ut}(\pi; s, \epsilon, s', \epsilon')$, over randomized trajectories. 
\begin{equation}\small
    \begin{aligned}
    \min_{\pi} \underset{(s, \epsilon, s',\epsilon') \sim \rmZ}{\E} \left[ \ell_{ut}(\pi; s, \epsilon, s', \epsilon') \right].
    \end{aligned}
\end{equation}
In the case of DQN, $\ell_{ut}(\cdot)$ is the mean squared Bellman error with noisy observations, as analogous to Equation~\ref{eq:MSBE_dqn}.
Therefore, the utility loss is defined as:
\begin{equation}\small\label{eq:utility_loss}
    \begin{aligned}
    &\ell_{ut}(\pi; s, \epsilon, s', \epsilon') = \\
    &\left[ Q_{\pi}(s+\epsilon,a) - \left( \gR(s, a) + \gamma \max_{a'} Q_{\pi^*}(s'+\epsilon', a') \right) \right]^2,
    \end{aligned}
\end{equation}
where $\epsilon,\, \epsilon' \sim\, \N(0, \sigma^2I)$, $a=\argmax_{i}\pi(s+\epsilon)_i$, and $\pi^*$ is the targeted optimal Q-network as defined in Section~\ref{subsec:drl_background}.

Unlike Equation~\ref{eq:MSBE_dqn}, the error here is computed not only over the state observations and actions, but also over the smoothing noise draws $\epsilon$ and $\epsilon'$. 
This loss function actually recovers the Gaussian-augmented training of DRL agents in the previous papers~\cite{kumar2021policy,wu2022crop}.
However, it does not directly address the robustness of the agent. 
Therefore, we will proceed to the next sections to explain how the certified radius can be enhanced by constructing specialized training objectives.

\subsection{Certified Radius Maximization}\label{subsec:radius_max}  
According to the problem definition in Section~\ref{sec:problem_statement}, the crux of the radius maximization involves two key aspects.
First, we need to build the dependence between the certified radius and the Q-network parameters $\pi$.
Second, a differentiable loss of the certified radius with respect to $\pi$ should be devised.

We begin with a lemma, which is a generalization from Lemma 2 in SmoothAdv~\cite{salman2019provably} and further proved in Lemma 2 of CROP~\cite{wu2022crop}.
The lemma is restated as follows.
\begin{restatable}
[\textit{Lipschitz continuity of smoothed return function}]{lemma}{lipschitz}\label{lemma:lipschitz}
For any measurable function $F_{\pi}: \rvz \in (\sS\times\sA\times\R^d)^T \rightarrow [A, B]$, let $\E [F_{\pi}(\rvz)] = \E_{\epsilon\sim\N(0,\sigma^2I)}{F_{\pi}(\rvz + \varepsilon)}$. There is $\rvz \rightarrow \E [F_{\pi}(\rvz)]$ is $\frac{(B-A)}{\sigma}\sqrt{2/\pi}$-Lipschitz.
\end{restatable} 
\noindent
The Lipschitz continuity immediately implies that $\frac{\sigma}{(B-A)\sqrt{2/\pi}} (\E[F(\rvz)] - \xi)$ serves as a certified radius, albeit a loosely certified one.
Although it does not directly bound the radius defined in Equation~\ref{eq:substitute_radius}, this soft radius is positively correlated with the previous radius. 
Specifically, when $\rvr_1$ is fixed, both radii monotonically increase with $\E [F_{\pi}(\rvz)] - \xi$.
Fortunately, this implication also suggests that this soft radius can be used as a surrogate for optimization. 
Building on the above lemma, we now proceed to the main theorem of deriving a differentiable soft certified radius.
\begin{restatable}
[\textit{Soft certified radius}]{theorem}{softradius}\label{theorem:softradius}
Given a target expected return $\xi \leq \E [F_{\pi}(\rvz)]$ where $F_{\pi}: \rvz \in (\sS\times\sA\times\R^d)^T \rightarrow [A, B]$, suppose an adversary perturbs the observed states by applying $\Delta=(\delta_0,\delta_1,...,\delta_{T-1})$.
The target expected return will not drop below $\xi$ if the perturbations satisfy:
\small
\begin{equation}\label{eq:soft_radius}
     \|\Delta\|_2 \leq \sigma\ \left[ \Phi^{-1}\left(\frac{\E [F_{\pi}(\rvz)]-A}{B-A}\right) - \Phi^{-1}\left(\frac{\xi-A}{B-A}\right) \right].
\end{equation}
\normalsize
\end{restatable}
\begin{proof}
Based on Lemma~\ref{lemma:lipschitz}, the mapping $\rvz \rightarrow \frac{\E[F_{\pi}(\rvz)] - A}{B - A}$ is $\sqrt{(2/\pi)}/\sigma$-Lipschitz.
Moreover, based on the same proof of Lemma 2 of Salman et al.~\cite{salman2019provably}, $\Phi^{-1}(\frac{\E[F_{\pi}(\rvz)] - A}{B - A})$ is $\frac{1}{\sigma}$-Lipschitz. 
We have
\small
\begin{equation}
\begin{aligned}
     \sigma\ [\Phi^{-1}(\frac{\E[F_{\pi}(\rvz)] - A}{B - A}) - \Phi^{-1}(\frac{\E[F_{\pi}(\rvz+\Delta)] - A}{B - A}) ]
     \leq  \|\Delta\|_2.  \\
\end{aligned}
\end{equation}
\normalsize
Thus, if the perturbation $\Delta$ satisfies:
\small
\begin{equation}
\begin{aligned}
    \|\Delta\|_2 \leq \sigma\ \left[ \Phi^{-1}(\frac{\E[F_{\pi}(\rvz)] - A}{B - A}) - \Phi^{-1}(\frac{\xi-A}{B-A}) \right],
\end{aligned}
\end{equation}
\normalsize
by the monotonicity of $\Phi^{-1}(\cdot)$, we obtain
\small
\begin{equation}
    \E[F_{\pi}(\rvz+\Delta)] \geq \xi.
\end{equation}
\normalsize
\end{proof}

This soft radius is computed through the expected return over clean trajectories, which means it can be efficiently sampled and calculated as a global radius for the $T$-step MDP.
Since Theorem~\ref{theorem:softradius} characterizes the globally certified radius for a trajectory of length $T$, it is necessary to decompose the global radius into local certified radii obtained at each time $t$.

The global certified radius can be decomposed into per-step action selection errors, which can, in turn, be bounded by a local robustness radius.
Specifically, since the radius positively correlates with $\E[F_{\pi}(\rvz)] - \xi$, given a fixed $\xi$, maximizing the soft radius is equivalent to increasing the expected return over randomized state-action trajectories. 
It is important to note that for discrete action spaces, if the optimal greedy policy is derived from a Q-network with parameters $\pi^*$, it is expected to select the action $a_t$ that maximizes the expectation of the action-value function at each time $t$.
In this context, the probability of a perturbation at step $t$ reducing the optimal return is equivalent to the probability that the perturbation causes the predicted action to shift from correct to incorrect.
We can thus formulate the optimal expected action-value function based on this optimal policy as $\bar{Q}_{\pi^*}(s_t, a_t) = \max_{\pi}\E_{\epsilon_t}Q_{\pi}(s_t+\epsilon_t, a_t)$.
Next, we have the following theorem:
\begin{restatable}
[\textit{Local certified radius}]{theorem}{localradius}\label{theorem:localradius}
Let $l_i$ and $u_i$ denote the lower and upper bounds of the expected action-value function at step $i$, respectively.
Let the perturbations from step $t$ to step $T-1$ be $\{\delta_i\}_{i=t}^{T-1}$.
Under a greedy policy, the optimal reward at step $t$ is obtained by taking action $a^{(1)}_i = \argmax_{a_i} \bar{Q}_{\pi^*}(s_i, a_i)$.
The optimal expected return from step $t$ to $T-1$ will not be reduced if the local perturbation $\delta_i$ at each step $i$ satisfies:
\small
\begin{equation}\label{eq:localradius}
    \begin{aligned}
        \|\delta_i\|_2 & \leq \frac{\sigma}{2} \left[  \Phi^{-1}\left( \frac{\bar{Q}_{\pi^*}(s_i, a^{(1)}_i) - l_i }{u_i - l_i} \right) \right.\\
         & - \left. \Phi^{-1}\left( \frac{\bar{Q}_{\pi^*}(s_i, a^{(2)}_i) - l_t }{u_i - l_i} \right)  \right],
    \end{aligned}
\end{equation}
\normalsize
where $a^{(2)}_i$ is the runner-up action $a^{(2)}_i = \argmax_{a_i:a_i\neq a^{(1)}_i} \bar{Q}_{\pi^*}(s_i, a_i)$.
\end{restatable}
\noindent
\begin{proof}[Proof sketch]
The proof stems from the Lipschitz continuity of $\bar{Q}_{\pi^*}(s_t, a_t)$ with respect to perturbations in the observed states.
Since $\bar{Q}_{\pi^*}(s_t, a_t)$ is Lipschitz continuous, by comparing the $\delta_t$-perturbed $\bar{Q}_{\pi^*}(s_t+\delta_t,a_t)$ with the unperturbed $\bar{Q}_{\pi^*}(s_t,a_t)$ for different actions, and using the monotonicity of $\Phi^{-1}(\cdot)$, we can determine the $\ell_2$ norm of the maximum allowable perturbation such that the action inducing the maximal reward remains selected by the greedy policy at each step $t$.
\end{proof}

This $\ell_2$ norm represents a local radius within which the perturbation does not compromise the optimal expected Q-value obtained by the optimal policy.
Preserving the optimal expected return can therefore be converted to a problem of increasing the gap between the top-1 and runner-up Q-values under the optimal policy.
We refer to this gap as the Q-gap in subsequent discussions.
The full proof of the theorem is in Appendix~\ref{append:proof}.
The local certified radius at each step can be obtained by iteratively applying Theorem~\ref{theorem:localradius} from $t=0$ to $t=T-1$.
Another implication of Theorem~\ref{theorem:localradius} is that the overall robustness budget $\E[F_{\pi}(\rvz)] - \xi$ can be distributed across per-step budgets, reflecting the errors in the actions.

Motivated by Theorem~\ref{theorem:localradius}, maximizing the global radius amounts to minimizing the occurrence of per-step errors in policy predictions.
From Equation~\ref{eq:localradius} in Theorem~\ref{theorem:localradius}, it follows that maximizing the local certified radius hinges on increasing the gap between the top-1 and the runner-up Q-values under randomized observations.
We can thus define a robustness loss such that the per-step action error is minimized:
\begin{equation}\small
\begin{aligned}
    & \ell_{ro}(\pi, \tilde{\pi}; s, \epsilon) = \lambda \cdot \mathds{1}_{\{Q_{\tilde{\pi}}(s+\epsilon, a^{(1)})\geq Q_{\tilde{\pi}}(s+\epsilon, a^{(2)})\}} \\
    & \max \{ 0, \eta - [ Q_{\pi}(s+\epsilon, a^{(1)}) - Q_{\pi}(s+\epsilon, a^{(2)} ) ] \}, \\
    & a^{(1)} = \argmax_a \pi(s+\epsilon)_a,\\
    & a^{(2)} = \argmax_{a: a\neq a^{(1)}} \pi(s+\epsilon)_a.
\end{aligned}
\end{equation}
Herein, $\eta$ and $\lambda$ represent the hinge loss offset and the robustness loss coefficient, respectively.
$s$ denotes an observation, which may be deterministic or randomized.
$a^{(1)}$ and $a^{(2)}$ are the predicted top-1 and runner-up actions in the action space $\sA$.
Note that the predicted action may differ from the optimal $a^{(1)}$ and $a^{(2)}$ required by Theorem~\ref{theorem:localradius}.
To mitigate this, $\tilde{\pi}$ is introduces as a \textit{reference network} to check whether $a^{(1)}$ indeed yields a higher reward than $a^{(2)}$.
The details of $\tilde{\pi}$ will be discussed in the following section.
This robustness loss regularizes policy training to maximize the gap between the Q-value induced by the top-1 action and that obtained by the runner-up action.

\begin{mdframed}[backgroundcolor=grey!10,rightline=true,leftline=true,topline=true,bottomline=true,roundcorner=1mm,everyline=false,nobreak=false]
\noindent \textbf{Remark.~}
Our certified radii differ from existing approaches in several aspects. 
First, they are functions of optimizable variables instead of a fixed $\tau$ determined via grid search. 
Second, by framing the total robustness cost as the cumulative return loss from per-step action errors, we convert the global certified radius into a local certified radius linked to the Q-gap.
Maximizing the Q-gap at each step reduces action errors, thereby enhancing the global certified radius.
\end{mdframed}

\subsection{Policy Imitation for \texttt{CAMP} Training}  
Given the objectives of improving both utility and robustness, the training objective of \texttt{CAMP} can be defined as minimizing the sum of the utility and robustness losses.
The remaining task is to determine the reference network $\tilde{\pi}$.
We found that attaching the robustness loss to the TD loss during training is unsuitable, as it leads to overestimation of Q-values.
Overestimating Q-values is a well-known issue that leads to sub-optimal convergence~\cite{yang2022rorl}.
Moreover, it is infeasible to constrain the maximal Q-value of the Q-network since the oracle Q-value is unknown. 
This dilemma motivates us to first approximate the oracle Q-value and use the approximated value as a reference to calibrate the Q-network.

To this end, we propose \textit{policy imitation} to stabilize the training process.
Specifically, we model the normal Q-function with a reference Q-network whose parameters are denoted as $\tilde{\pi}$.
$\tilde{\pi}$ is trained using the utility loss $\ell_{ut}(\tilde{\pi}; s, \epsilon, a, s', \epsilon')$.
The primary Q-network then mimics the behavior of this reference network to improve its utility while the robustness loss is applied for robustness enhancement. 
In this way, the reference network helps the primary Q-network predict more accurate Q-values. 
The imitation loss is defined as the cross-entropy loss between $Softmax(\pi(s+\epsilon))$ and $Softmax(\tilde{\pi}(s+\epsilon))$:
\begin{equation}\label{eq:imitation_loss}\small
\begin{aligned}
    \ell_{im}(\pi, \tilde{\pi}; s, \epsilon) = - \sum_{i=1}^{|\sA|} Softmax(\tilde{\pi}(s+\epsilon))_i \log Softmax(\pi(s+\epsilon))_i.
\end{aligned}
\end{equation}
Therefore, $\ell_{im}(\pi, \tilde{\pi}; s, \epsilon)$ measures the difference between the actions predicted by $\pi$ and $\tilde{\pi}$, given the same observation input $s+\epsilon$.

The robustness loss and the imitation loss take the same randomized observation $s+\epsilon$ as used in the utility loss for computation.
The final loss function of \texttt{CAMP} is as follows:
\begin{equation}\label{eq:CAMP_loss}\small
\begin{aligned}
    \ell_{CAMP}(\pi, \tilde{\pi}; s, \epsilon, s', \epsilon') & = \ell_{ut}(\tilde{\pi}; s, \epsilon, s', \epsilon') \\
    &+ \ell_{ro}(\pi, \tilde{\pi}; s, \epsilon) + \ell_{im}(\pi, \tilde{\pi}; s, \epsilon).
\end{aligned}
\end{equation}
Let $\gZ$ denote a replay buffer or a dataset of trajectories. 
The overall training objective is then defined as follows:
\begin{equation}\small\label{eq:camp_objective}
    \min_{\pi, \tilde{\pi}} \underset{(s, \epsilon, s', \epsilon') \sim \gZ}{\E} \ell_{CAMP}(\pi, \tilde{\pi}; s, \epsilon, s', \epsilon').
\end{equation}

In our implementation, we found that using two separate replay buffers for $\pi$ and $\tilde{\pi}$ can enhance the agent's performance and stability.
Consequently, $\pi$ and $\tilde{\pi}$ operate alternately in the same environment, storing the collected trajectories into two distinct replay buffers, $\gZ$ and $\tilde{\gZ}$.
The training algorithm samples trajectories from $\tilde{\gZ}$ to minimize the robustness loss and the imitation loss, while sampling from $\gZ$ to minimize the utility loss. 

To summarize the practical training steps of \texttt{CAMP}, the complete training procedure is presented in Algorithm~\ref{alg:CAMP_train}.
Note that $d_t$, $d_j$ and $\tilde{d}_j$ are binary variables representing whether the current states (\ie $s_t$, $s_j$, and $\tilde{s}_j$) are terminal states.
If yes, the Q-function should show that the agent gets no additional rewards after the current state.
Moreover, only the derivatives of $\ell_{ut}$ with respect to the reference network parameters are calculated in line 25.
Conversely, since $\ell_{ut}$ is independent of the primary Q-network, only $\ell_{ro}$ and $\ell_{im}$ induce gradients for optimizing the primary network in line 26.
We will validate the effectiveness of \texttt{CAMP} through experiments in the following section.

\begin{algorithm}[t]
\footnotesize
\caption{\texttt{CAMP} Training}\label{alg:CAMP_train}
\KwIn{Replay buffers $\gZ$ and $\tilde{\gZ}$, environment $\tE$ with state transition function $\gT$ and reward function $\gR$, discount factor $\gamma$, primary network $\pi$, reference network $\tilde{\pi}$, noise scale $\sigma$, polyak rate $k$, burn in step $t_b$, batch size $N$, learning rate $\alpha$, target update frequency $\bigtriangleup$.}
\KwOut{Trained $\pi$}
Initialize $s_0$\\
Initialize $\pi$, $\tilde{\pi}$\\
Initialize reference target network $\tilde{\pi}'\ \gets\ \tilde{\pi}$\\
\For{$t \in \{0, 1, 2,..., T-1\}$}
{   
    $\epsilon_t, \epsilon_{t+1} \sim \N(0, \sigma^2 I)$\\
    $\hat{s}_t\ \gets\ s_t + \epsilon_t$\\
    \If{$t \leq t_b$}
    {   
        $a_{t}\ \gets\ \textsc{RandomSample}(\sA) $\\
        $s_{t+1},\ d_t, r_{t+1}\ \gets \gT({s}_{t}, a_{t}),\ \gR({s}_{t}, a_{t})$ \\
        \If{$t \% 2 = 0$}
        {
            $\gZ\ \gets\  (s_{t}, \epsilon_t, s_{t+1}, \epsilon_{t+1}, d_t)$ \\
        }
        \Else
        {
            $\tilde{\gZ}\ \gets\  (s_{t}, \epsilon_t, s_{t+1}, \epsilon_{t+1}, d_t)$ \\
        }
    }
    \Else 
    {   
        \If{$t \% 2 = 0$}
        {
            $ a_{t}\ \gets\ \argmax_a \pi(\hat{s}_t)_a $ \\
            $s_{t+1},\ d_t, r_{t+1}\ \gets \gT({s}_{t}, a_{t}),\ \gR({s}_{t}, a_{t})$ \\
            $\gZ\ \gets\  (s_{t}, \epsilon_t, s_{t+1}, \epsilon_{t+1}, d_t)$ \\
        }
        \Else
        {
            $ a_{t}\ \gets\ \argmax_a \tilde{\pi}(\hat{s}_t)_a$ \\
            $s_{t+1},\ d_t, r_{t+1}\ \gets \gT({s}_{t}, a_{t}),\ \gR({s}_{t}, a_{t})$ \\
            $\tilde{\gZ}\ \gets\  (s_{t}, \epsilon_t, s_{t+1}, \epsilon_{t+1}, d_t)$ \\
        }
    }
    $\{(s_j,\ \epsilon_j, \ s'_j,\ \epsilon'_j,\ d_j)\}_{j=1}^{N}\ \gets\ \textsc{BatchSample}(\gZ, N)$ \\
    $\{(\tilde{s}_j,\ \tilde{\epsilon}_j,\ \tilde{s}'_j,\ \tilde{\epsilon'}_j,\ \tilde{d}_j)\}_{j=1}^{N}\ \gets\ \textsc{BatchSample}(\tilde{\gZ}, N)$ \\
    $\tilde{\pi}\ \gets\ \tilde{\pi} - \alpha (1 - \tilde{d}_j) \frac{d \sum_{j=1}^{N} \ell_{CAMP}(\pi, \tilde{\pi}; \tilde{s}_j, \tilde{\epsilon}_j, \tilde{s}'_j, \tilde{\epsilon}'_j )}{d\tilde{\pi}}$ \\
    $\pi\ \gets\ \pi - \alpha (1 - d_j) \frac{d \sum_{j=1}^{N} \ell_{CAMP}(\pi, \tilde{\pi}; s_j, \epsilon_j, s'_j, \epsilon'_j)}{d\pi}$ \\
    \If{$t \% \bigtriangleup == 0$}
    {
        $\tilde{\pi}'\ \gets\ k \tilde{\pi} + (1-k) \tilde{\pi}'$ \\
    }
}
\KwOut{$\pi$}
\end{algorithm}

\section{Experiments}
We evaluate the performance of \texttt{CAMP} in this section.
Our experiments aim to answer the following research questions:
\begin{itemize}[leftmargin=*]
    \item Can \texttt{CAMP} improve the certified expected returns under the same perturbation budget?
    \item Is \texttt{CAMP} a broadly applicable enhancement for various environments and Q-networks?
    \item How robust is \texttt{CAMP} towards changes in its hyper-parameters?
\end{itemize}
To answer these questions, we systematically compare \texttt{CAMP} to baselines in different environments and conduct an ablation study on the hyper-parameters.
We will first introduce the environments, DRL algorithms, Q-networks, and certification methods used in our experiments.

\noindent \textbf{Environments.~} 
We make evaluations and comparisons in five representative environments, namely \textit{Cartpole}, \textit{Highway}, \textit{Bank Heist}, \textit{Pong}, and \textit{Freeway}.
Cartpole simulates a classic problem in control dynamics.
The observed states in Cartpole are four kinematic signals representing the cart position, cart velocity, pole angle, and pole angular velocity. 
The action space contains two discrete actions (\ie push the cart to the left or right). 
Highway simulates an autonomous driving environment in which the observations are based on the kinematics of nearby cars and the actions are four control inputs.
Freeway, Pong and Bank Heist are Atari games that examine the performance of certification algorithms towards high-dimensional observations.
The observed states in both environments are frames of $84 \times 84$ greyscale images with pixel values in $[0, 255]$.
The available actions in Freeway, Pong, and Bank Heist consist of 3, 6, and 18 discrete options, respectively.
We adopt Pong with only one round (\ie Pong1r) in our experiments.

\noindent \textbf{DRL algorithms.~}
We evaluated our method based on DQN in various environments.
DQN is employed as the DRL algorithm for evaluation.
For CartPole and Highway, the architecture of the Q-network is a multi-layer perceptron (MLP) network.
On the other hand, Nature CNN is employed as the Q-network for Freeway, Pong, and Bank Heist~\cite{mnih2015human}. 
The detailed architectures are also attached in Appendix~\ref{appendsub:arc}.
These settings follow the convention in the evaluation of DQN, which endorses fair comparisons with previous works.

\noindent \textbf{Certification algorithms and baselines.~} 
We employ Policy-Smoothing (PS)~\cite{kumar2021policy} as the certification algorithm.
PS leverages the generalized Neyman-Pearson lemma towards the structure deterministic adversary and generates tight robustness certificates.
We apply the exact CDF smoothing as introduced in Section~\ref{subsec:cdf_smoothing} for agents obtaining continuous rewards in Highway, Freeway, and Bank Heist.
Following the PS paper, for agents obtaining 0/1 reward at each time step in Cartpole and Pong1r, we employ per-step point estimates of the CDF function based on the Clopper-Pearson method~\cite{cohen2019certified} and take the sum as the CDF.
This method is equivalent to CDF smoothing and is found to achieve slightly better results in practice.
We set two baselines by certifying using policies trained by Gaussian Augmentation (Gaussian)~\cite{kumar2021policy,wu2022crop} and NoisyNet~\cite{plappert2018parameter}.
Gaussian injects noises sampled from the same Gaussian distribution used for certification into the observations during training. 
On the other hand, NoisyNet randomizes the weights of the linear layers in the Q-network to aid efficient exploration during training and improve return obtained during tests.
We apply noise with the same scale to all the methods during training for comparison.

\noindent \textbf{Hyper-parameters and detailed settings.~}
The discount factor in all training experiments is set to $\gamma=0.99$ except for Highway, where $\gamma = 0.8$.
The agents randomly sample actions according to a greedy probability which decreases linearly from $1$ to $0$ within the first $16\%$ training steps.
In Highway, the probability decreases linearly from $1$ to $0.05$ over the first $10\%$ of the training steps.
The target networks are updated softly using a Polyak rate of $1$.
In \texttt{CAMP}, we use adaptive $\eta$ values by setting $\eta = \max_{(s,a)\sim \gZ_t} Q_{\pi}(s,a) - \min_{(s,a)\sim \gZ_t} Q_{\pi}(s,a)$, where $\gZ_t$ represents the trajectory batch at the $t$-th step.
We also conduct an ablation study in Section~\ref{subsec:ablation} to analyze the impact of training with different $\lambda$ values on the certification performance.
In the certification, the return is computed with $\gamma=1$.
Each certification run includes return values from $10000$ game rounds in which the observation is randomized by Gaussian noise. 

\begin{itemize}[leftmargin=*]
\item \textbf{Cartpole:~} 
We train agents in the Gymnasium implementation of \textit{CartPole-v0} with single-frame observations (Cartpole-1) and five-frame observations (Cartpole-5), respectively.
In both cases, we use a batch size of $1024$, a learning rate of $5\times 10^{-5}$, and the training spans $500$k steps.
The training stops when the return reaches $200$ (\ie the highest possible reward). 
The agent starts training after $10000$ burn-in steps, and the target network is updated every $10$ steps.
\item \textbf{Highway:~}
We select \textit{highway-fast-v0} as the environment.
We use a batch size of $1024$, a learning rate of $5 \times 10^{-5}$, and train for $1000$k steps.
The training starts after $1000$ steps, and the target network is updated every $10$ steps.
\item \textbf{Atari:~} 
We use Gymnasium implementations of \textit{FreewayNoFrameskip-v0} in hard mode, \textit{PongNoFrameskip-v0}, and \textit{BankHeistNoFrameskip-v4} as the implementations.
The training takes $10$ million steps with a learning rate of $1\times 10^{-4}$.
We observed that applying policy imitation at a later training stage can enhance training stability.
Therefore, in our experiments, we first train the reference network until convergence. 
Subsequently, we duplicate the reference network as the primary network and apply policy imitation with \texttt{CAMP} to update the primary network while keeping the reference network fixed.
The reference network is trained under the same setting as the baselines while updating the primary policy takes an extra $1$ million steps.
\end{itemize}
Detailed hyper-parameter settings are given in Appendix~\ref{appendsub:param_setting}.

\begin{figure*}[t]
     \centering
     \begin{subfigure}
         \centering
         \includegraphics[width=.32\linewidth]{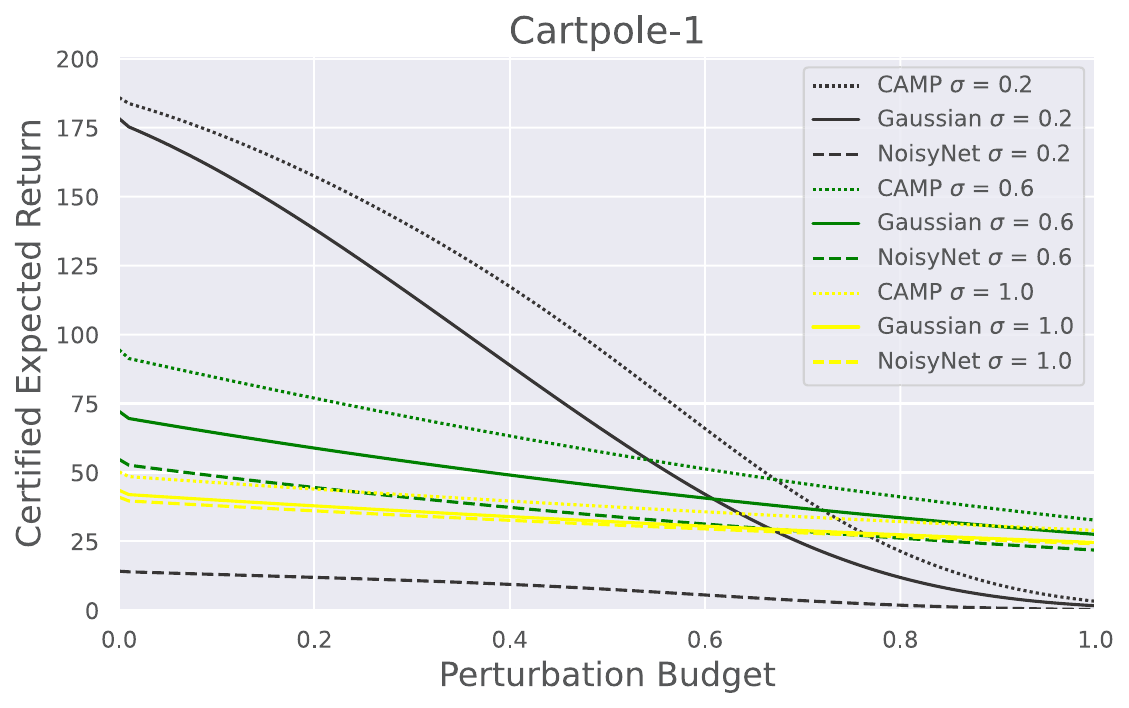}    
     \end{subfigure}
     \begin{subfigure}
         \centering
         \includegraphics[width=.32\linewidth]{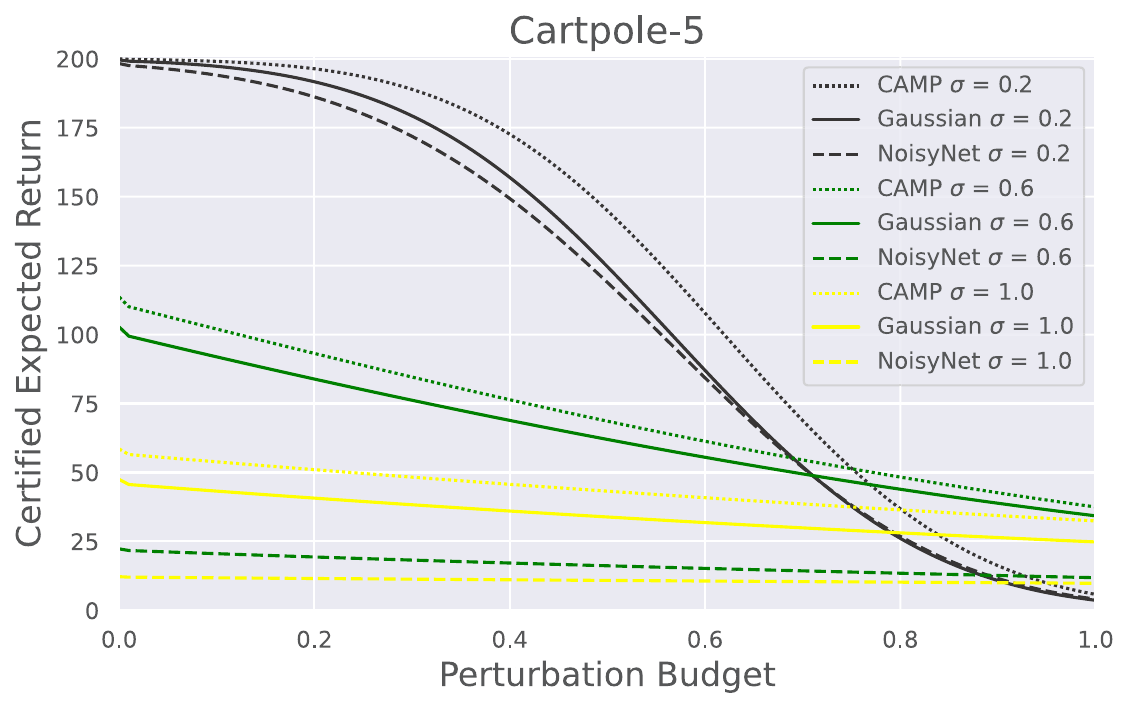}
     \end{subfigure}
     \begin{subfigure}
         \centering
         \includegraphics[width=.32\linewidth]{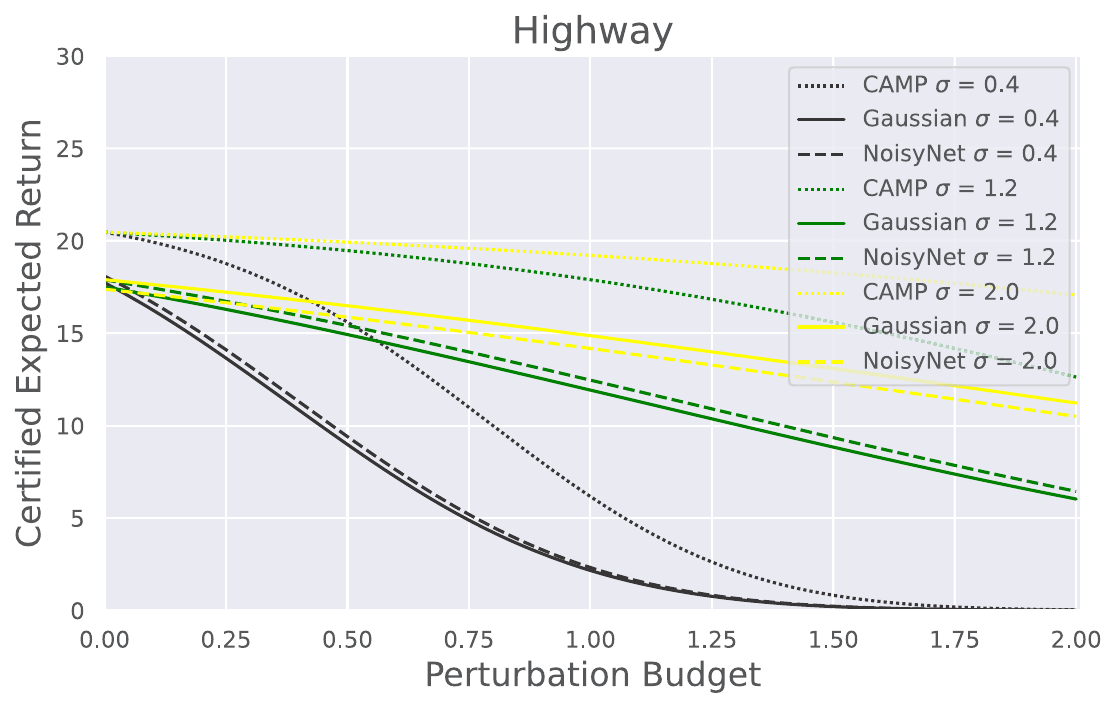}    
     \end{subfigure}
     \begin{subfigure}
         \centering
         \includegraphics[width=.32\linewidth]{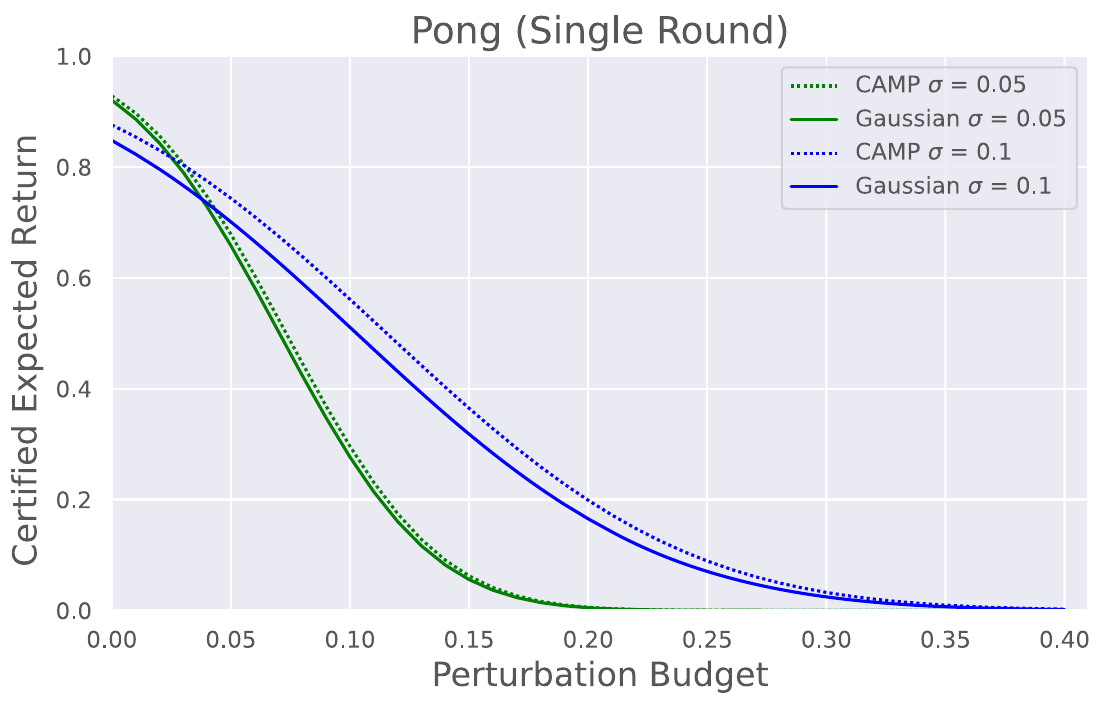}
     \end{subfigure}
     \begin{subfigure}
         \centering
         \includegraphics[width=.32\linewidth]{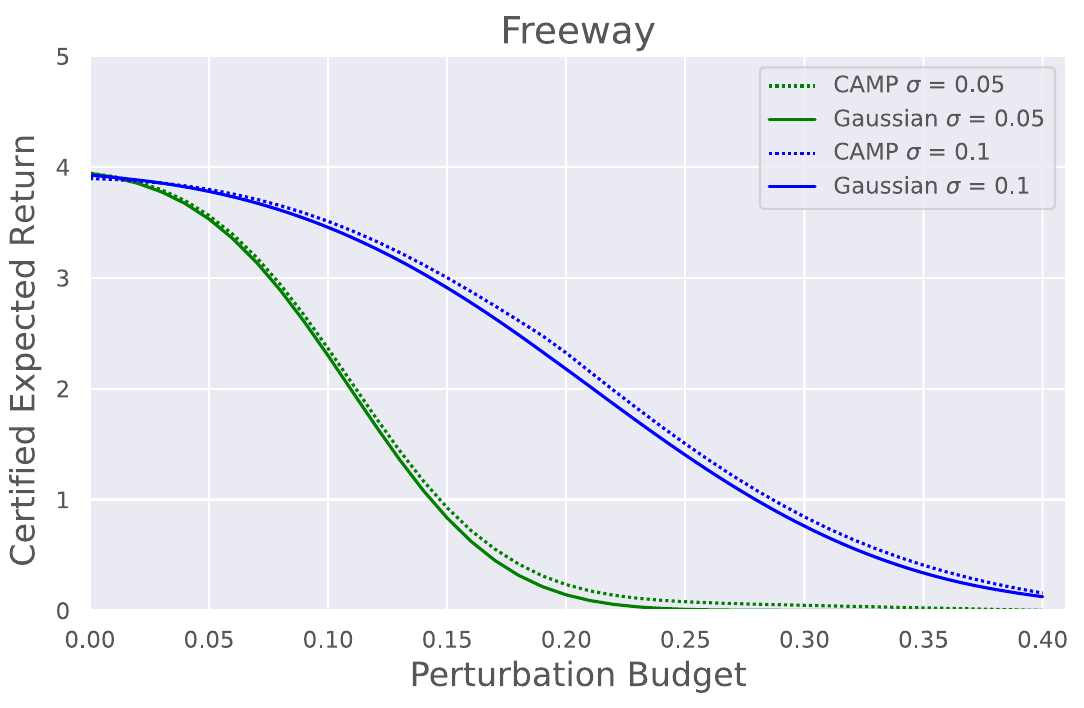}
     \end{subfigure}
     \begin{subfigure}
         \centering
         \includegraphics[width=.32\linewidth]{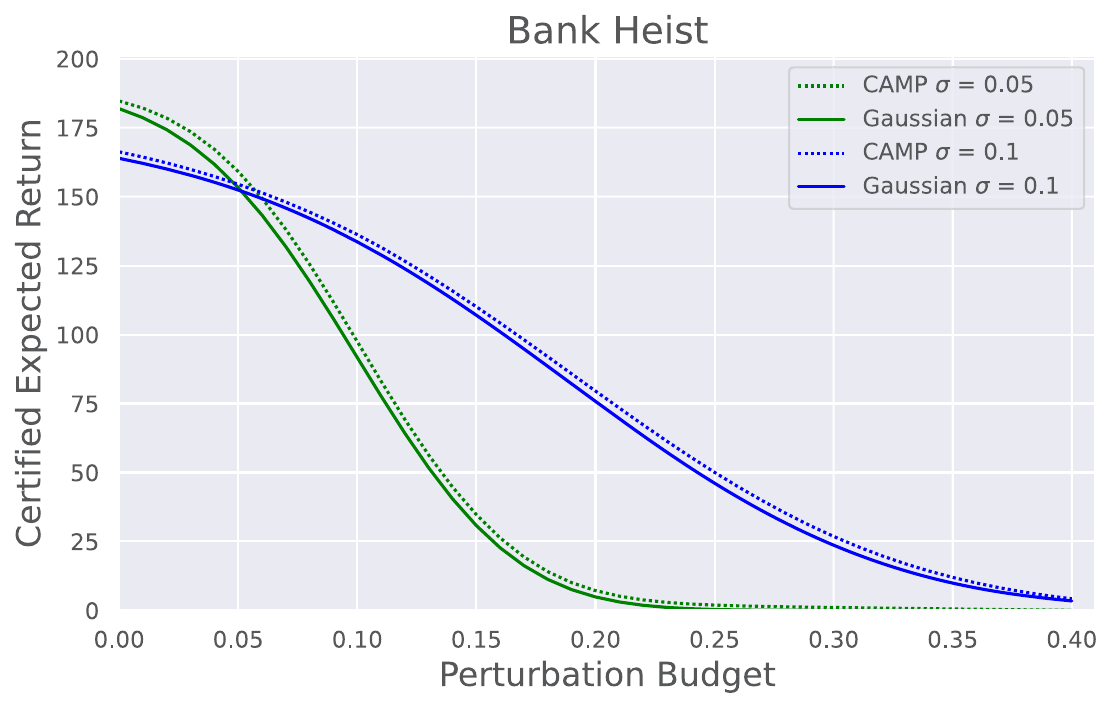}
     \end{subfigure}
\caption{Certification results on CartPole, Highway, Pong, Freeway, and Bank Heist. The perturbation budgets for Atari games (Freeway, Pong, and Bank Heist) are normalized by dividing by $255$.}
\label{fig:cert_comparison}
\end{figure*}

\begin{figure*}[t]
     \centering
     \begin{subfigure}
         \centering
         \includegraphics[width=.32\linewidth]{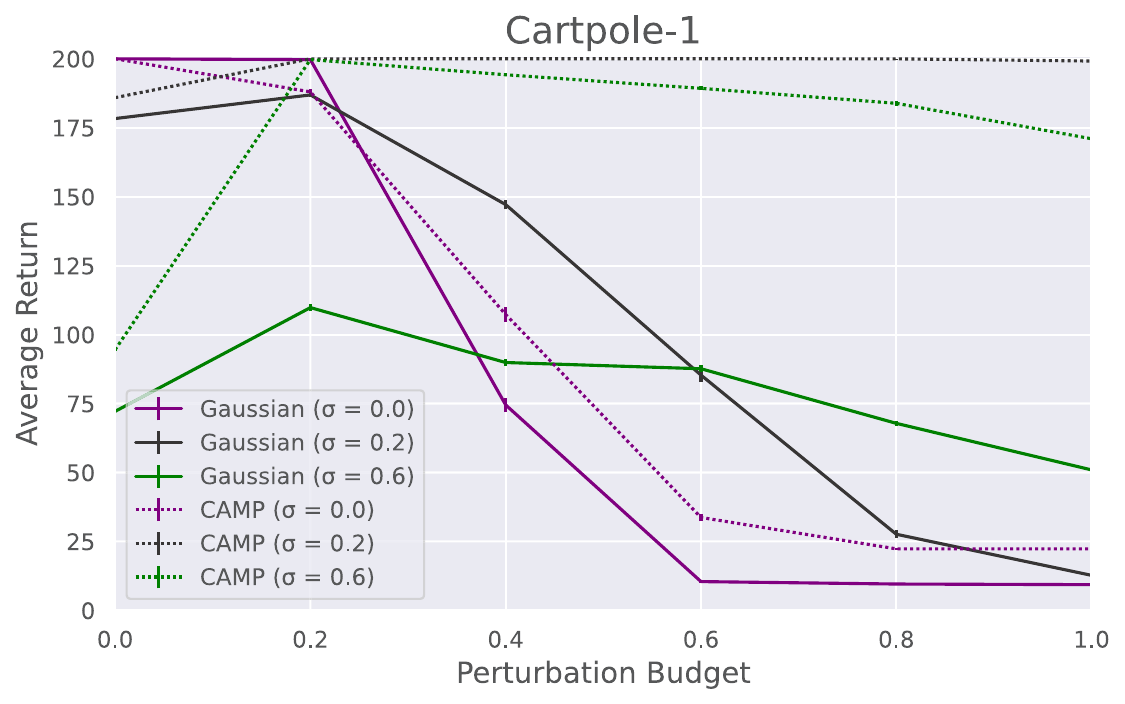}
     \end{subfigure}
     \begin{subfigure}
         \centering
         \includegraphics[width=.32\linewidth]{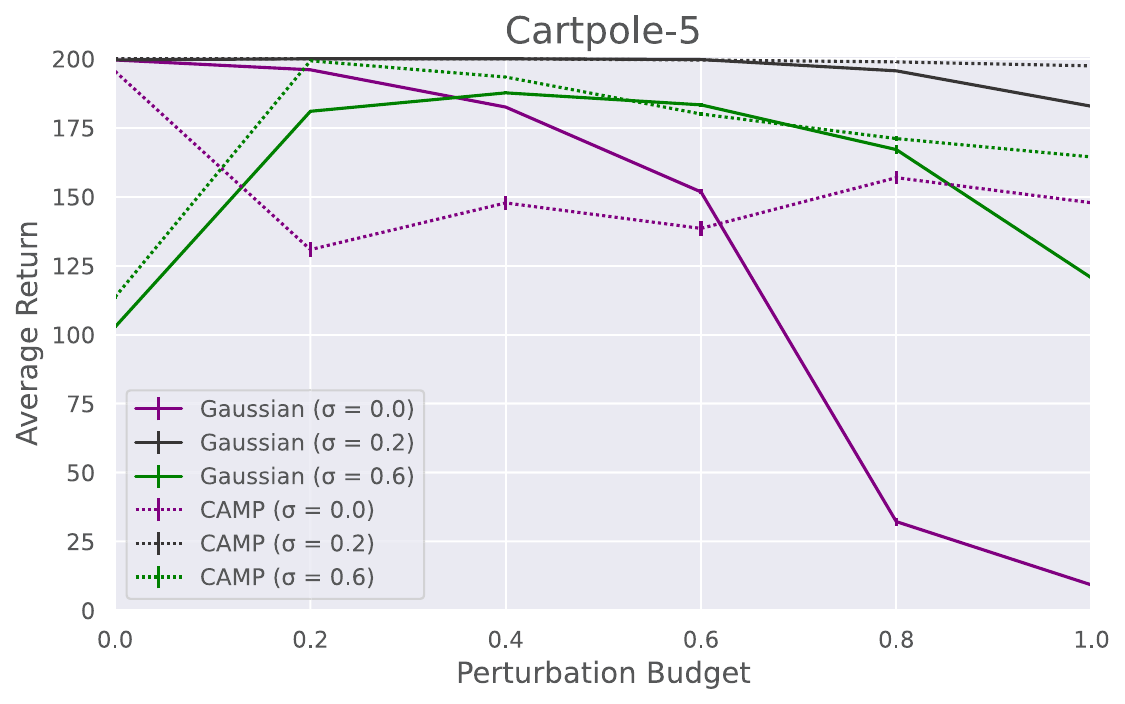}
     \end{subfigure}
     \begin{subfigure}
         \centering
         \includegraphics[width=.32\linewidth]{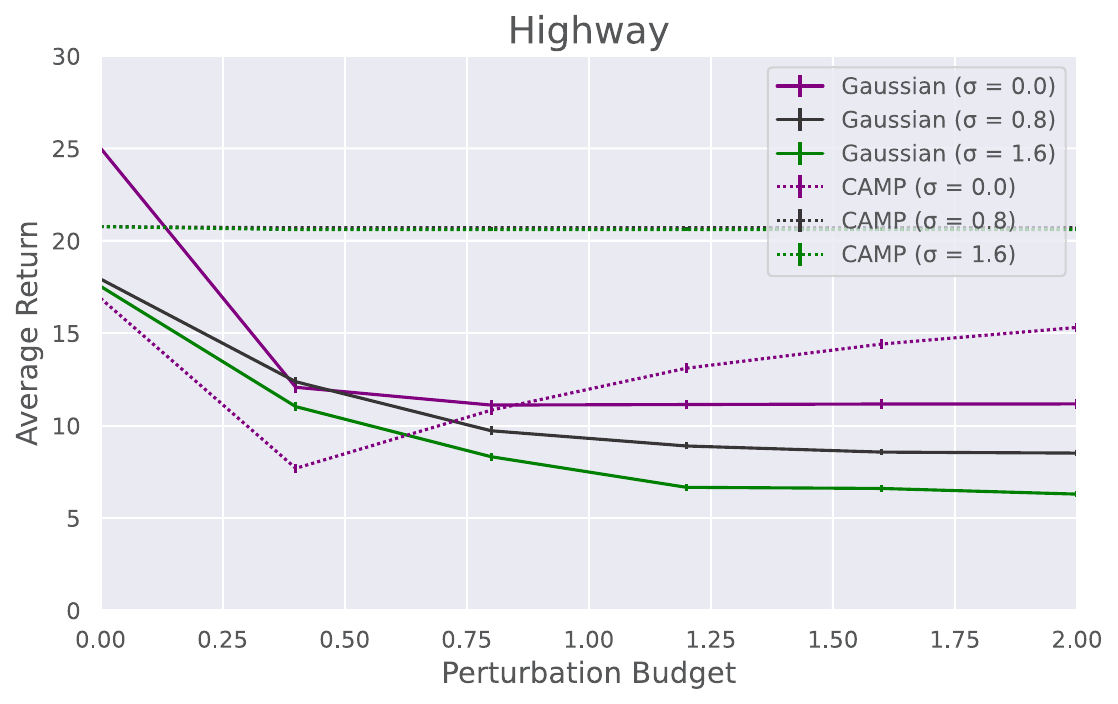}
     \end{subfigure}
     \begin{subfigure}
         \centering
         \includegraphics[width=.32\linewidth]{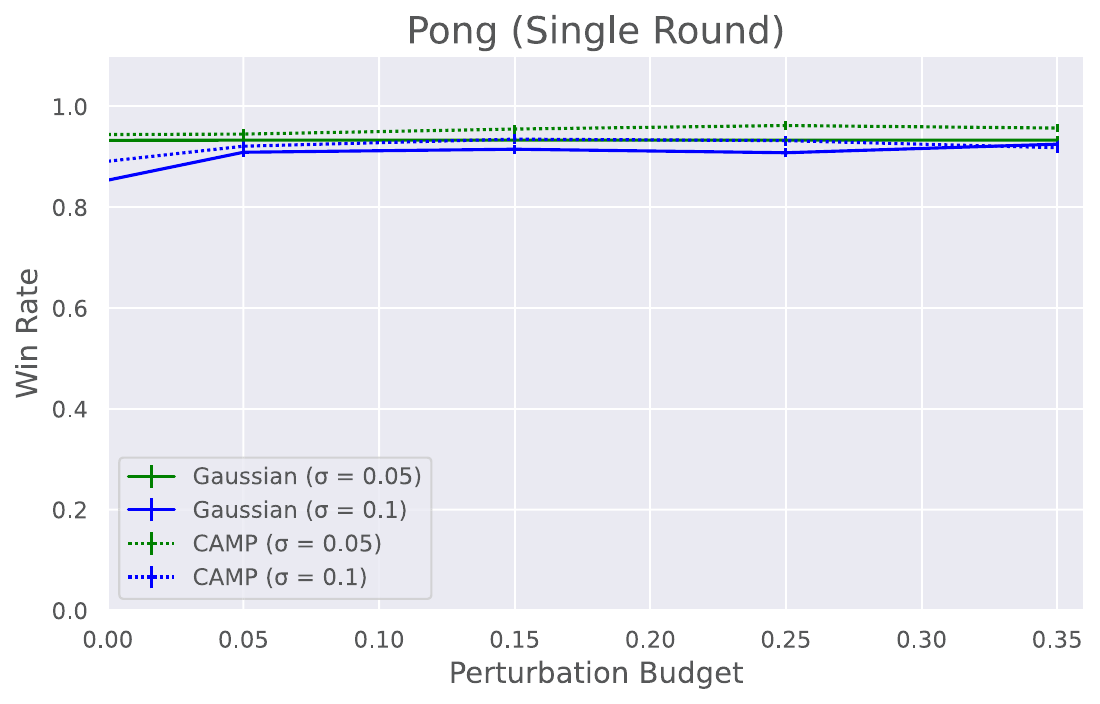}
     \end{subfigure}
     \begin{subfigure}
         \centering
         \includegraphics[width=.32\linewidth]{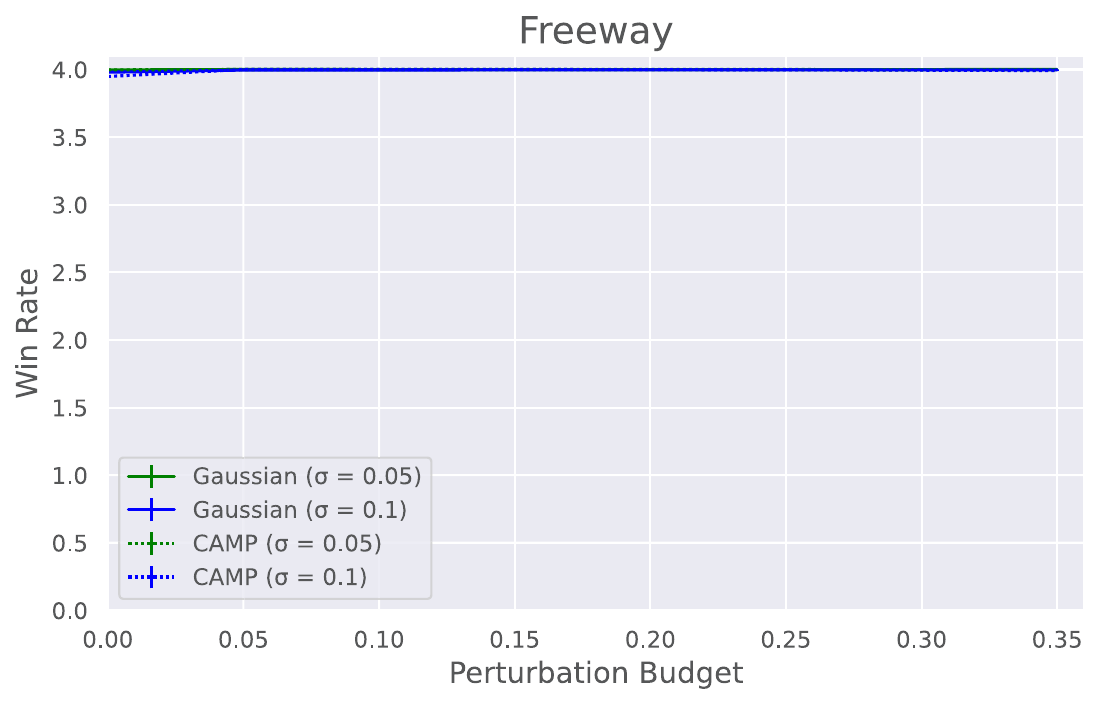}
     \end{subfigure}
     \begin{subfigure}
         \centering
         \includegraphics[width=.32\linewidth]{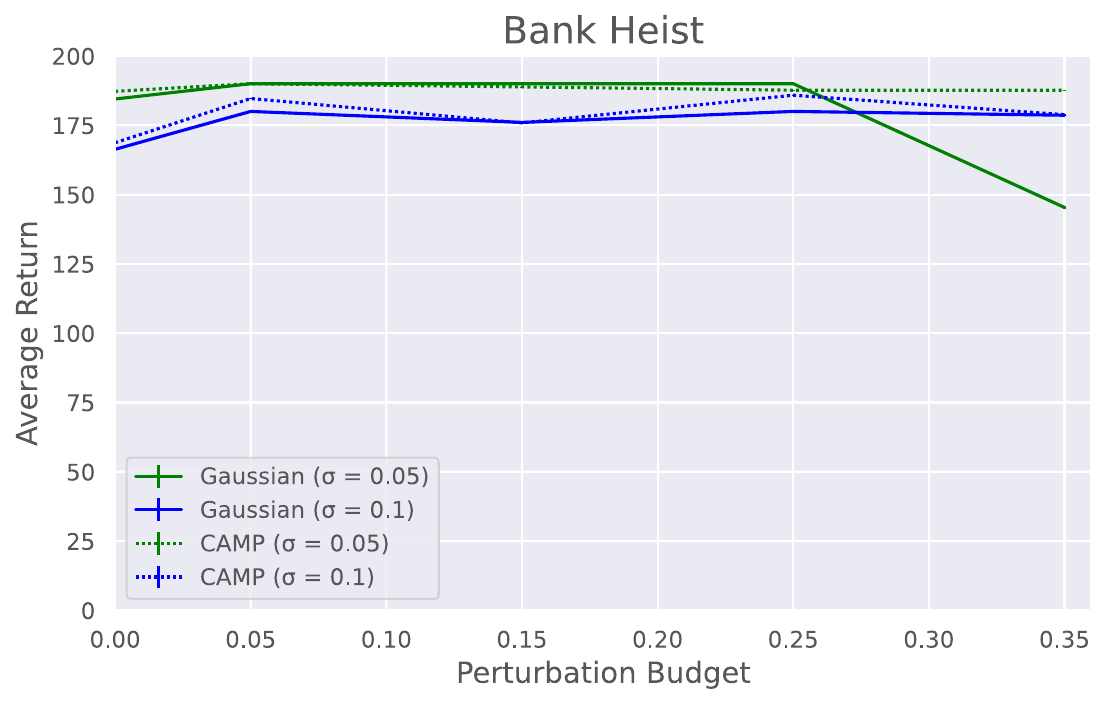}
     \end{subfigure}
\caption{Empirical robustness of agents against PGD in CartPole, Highway, Pong, Freeway, and Bank Heist. We use the same attack in PS~\cite{kumar2021policy} to evaluate the robustness of the agent in individual runs. The perturbation budgets for Freeway, Pong, and Bank Heist are normalized by dividing by $255$. In Pong, since agents either win or lose in each run, the expected return corresponds to the win rate.}
\label{fig:empirical_comparison}
\end{figure*}

\begin{figure*}[t]
     \centering
     \begin{subfigure}
         \centering
         \includegraphics[width=.32\linewidth]{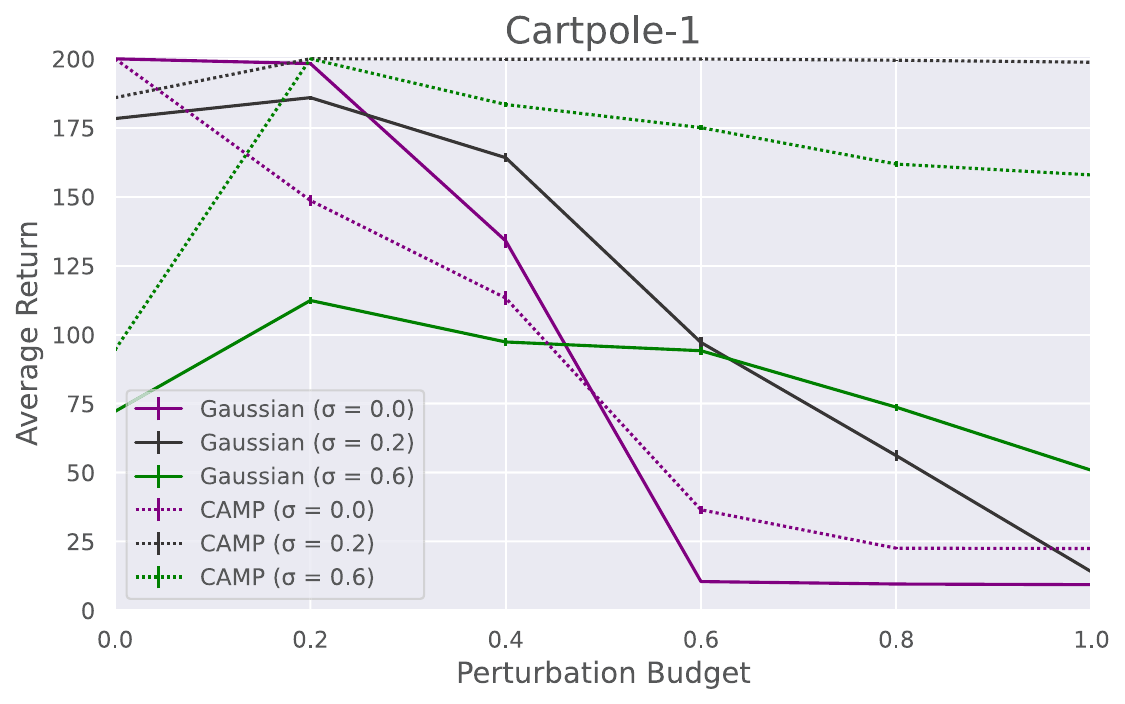}
     \end{subfigure}
     \begin{subfigure}
         \centering
         \includegraphics[width=.32\linewidth]{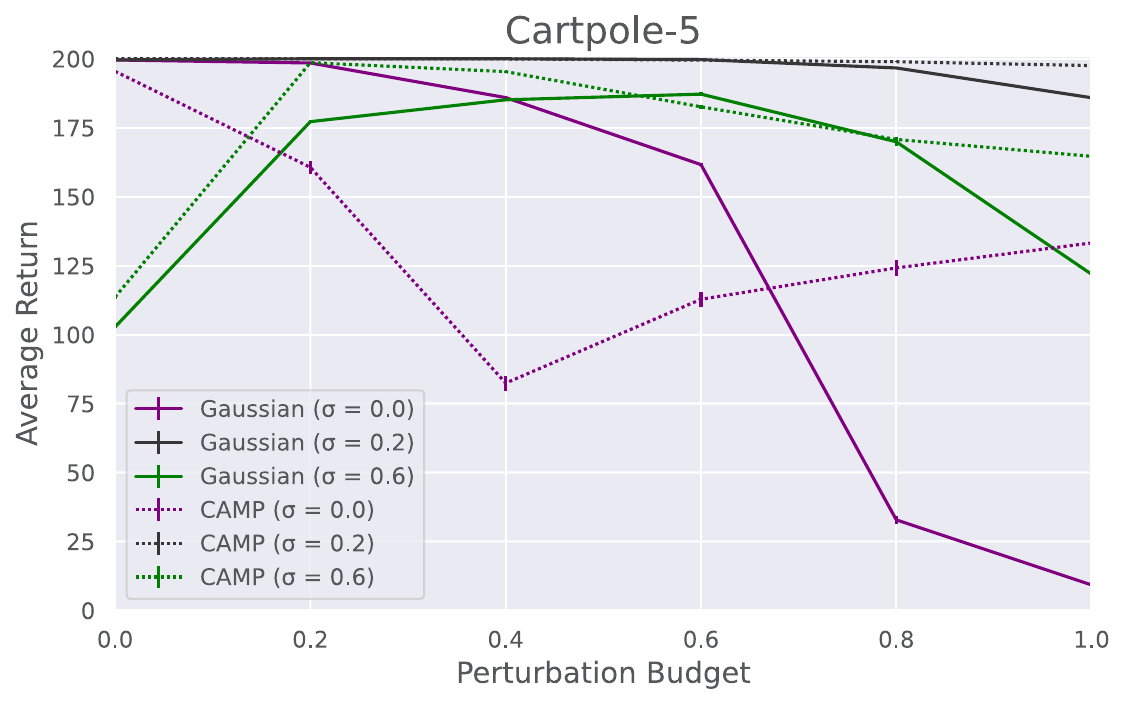}
     \end{subfigure}
     \begin{subfigure}
         \centering
         \includegraphics[width=.32\linewidth]{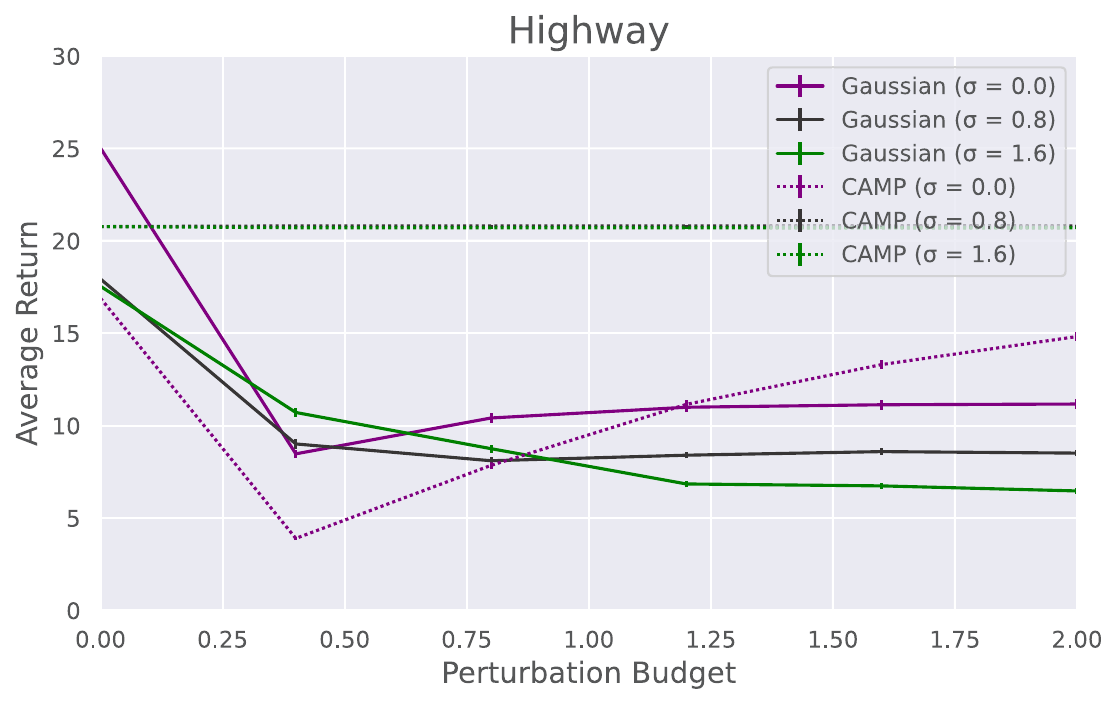}
     \end{subfigure}
     \begin{subfigure}
         \centering
         \includegraphics[width=.32\linewidth]{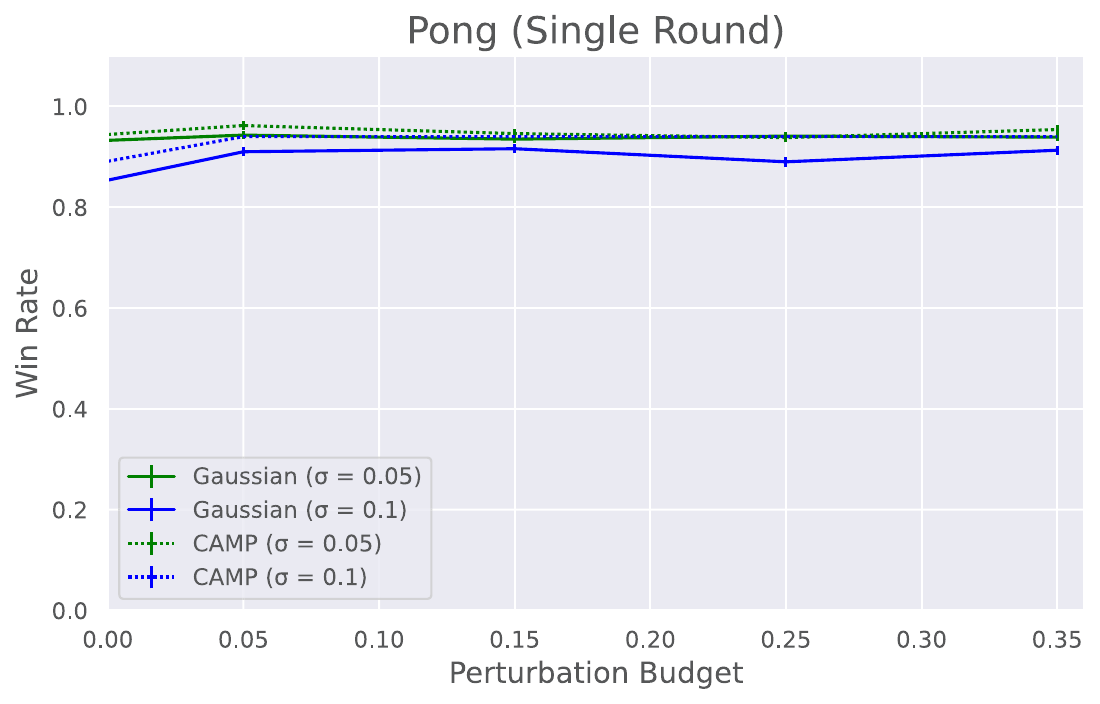}
     \end{subfigure}
     \begin{subfigure}
         \centering
         \includegraphics[width=.32\linewidth]{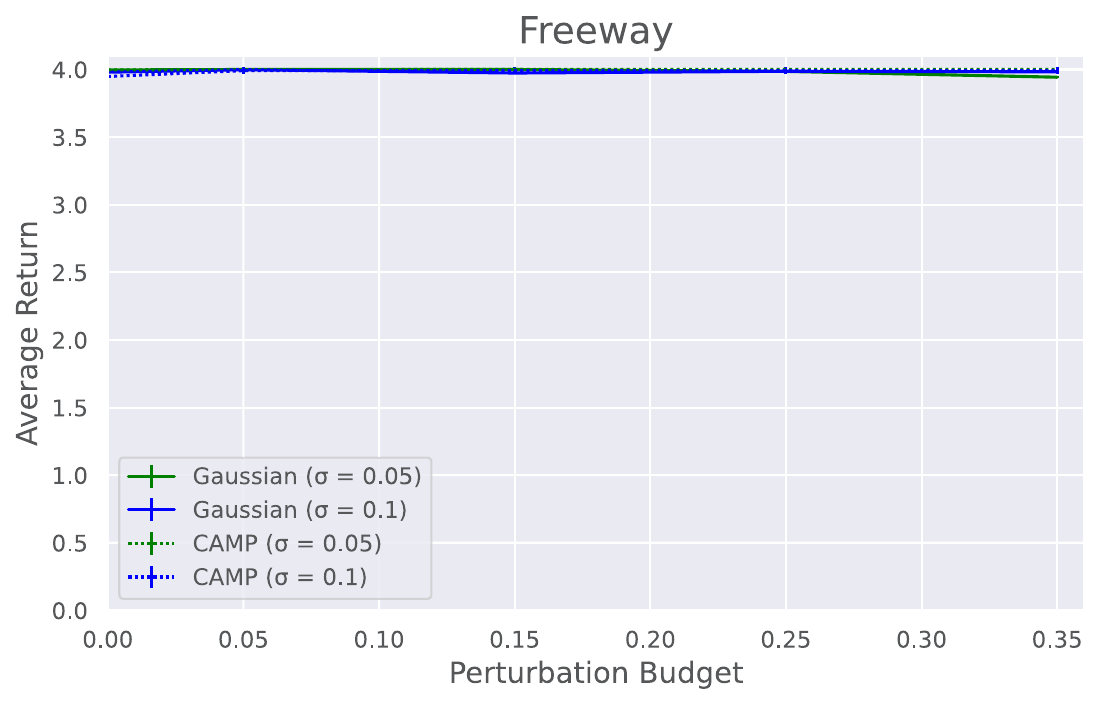}
     \end{subfigure}
     \begin{subfigure}
         \centering
         \includegraphics[width=.32\linewidth]{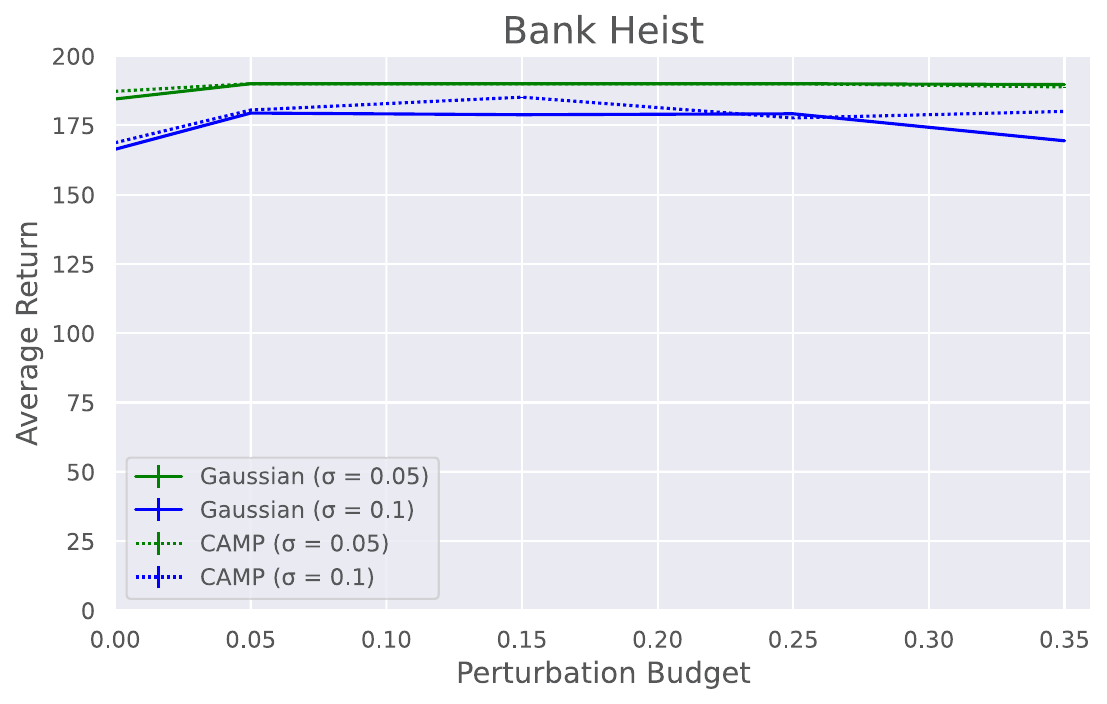}
     \end{subfigure}
\caption{Empirical robustness of agents against APGD in CartPole, Highway, Pong, Freeway, and Bank Heist. The average return values are evaluated under the same settings as PGD. APGD preserves the perturbation budget at each step, allowing it to perturb observations across more steps, which can result in more significant performance degradation for the agents.}
\label{fig:empirical_comparison_apgd}
\end{figure*}

\subsection{Evaluation}\label{subsec:comparison}
\noindent \textbf{Comparison of certification results.~}
We benchmark the certified reward of our method with that from baselines in this section.
We unify the hyper-parameters in the certification and apply them to agents trained by Gaussian, NoisyNet, and \texttt{CAMP}, respectively.
For clarity, in Cartpole-1 and Cartpole-5, we compare the certified returns when $\sigma\in\{0.2, 0.6, 1.0\}$ and plot them in Figure~\ref{fig:cert_comparison}.
Meanwhile, the full Cartpole certification results based on $\sigma\in\{0.2, 0.4, 0.6, 0.8, 1.0\}$ are provided in Figure~\ref{fig:full_cert_comparison} of Appendix~\ref{appendsub:full_certification_results}.
We compare all three methods in these two environments. 

Based on the results, the certified expected returns from \texttt{CAMP} show universal increases across all combinations of certified radii and smoothing noise levels compared to the baselines. 
In both Cartpole-1 and Cartpole-5, the increase in certified expected returns with \texttt{CAMP} becomes more pronounced when certifying at an $\ell_2$ radius between $0.2$ and $0.8$.
According to Figure~\ref{fig:cert_comparison}, in Cartpole-1, the utility gain from \texttt{CAMP} is slightly more obvious than that in Cartpole-5.
However, the most significant improvement is observed in Cartpole-5 at $\sigma=0.4$, where \texttt{CAMP} doubles the certified return against adversaries with budgets between 0.2 and 1.0. 
These results indicate that \texttt{CAMP} can effectively enhance certified returns at fixed certified radii in environments with low-dimensional observations.
NoisyNet performs worse than both Gaussian and \texttt{CAMP}, suggesting that NoisyNet may be ineffective in handling observations with significant noise. 
Notably, when $\sigma=0.2$ in Cartpole-1, NoisyNet fails in certification, likely due to underfitting caused by its randomized parameters, which complicates learning the Q-function in conditions with less informative observations.

Additionally, the results for the Highway environment with $\sigma\in\{0.4, 1.2, 2.0\}$ are shown in Figure~\ref{fig:cert_comparison}.
It can be observed that \texttt{CAMP} significantly outperforms the baselines across all smoothing noise levels.
These results suggest that \texttt{CAMP} agents are relatively robust in the Highway environment and can tolerate more adversarial perturbations compared to agents trained by the baseline methods.
The full comparison with $\sigma\in\{0.4, 0.8, 1.2, 1.6, 2.0\}$ is included in Appendix~\ref{appendsub:full_certification_results}.

On the other hand, we certify agents in Freeway, Pong1r, and Bank Heist with $\sigma \in \{12.75, 25.5\}$ to evaluate the effectiveness of \texttt{CAMP} in high-dimensional observation spaces. 
Improving certification performance in these high-dimensional environments is considered a more challenging task. 
We include Gaussian and \texttt{CAMP} in the comparison, as NoisyNet consistently underperformed Gaussian in our experiments when training in noisy environments.
Similarly, the certification results are in Figure~\ref{fig:cert_comparison}.
In the plots, we normalize the perturbation budget and $\sigma$ values by dividing by 255. 
It is evident that \texttt{CAMP} consistently certifies better expected returns than Gaussian. 
While the improvements are more subtle compared to other environments, they are particularly noticeable when the attack budget ranges from $[11.25, 63.75]$ in Pong, $[38.25, 102]$ in Freeway, and $[0, 76.5]$ in Bank Heist.

Given these certification results on expected returns, we proceed to examine the empirical robustness of return values in each game round against adversarial perturbations.

\noindent \textbf{Empirical robustness verification.~}
In addition to comparing certification results, we also explore the empirical robustness of the trained DRL agent in each game episode. 
Recall that certification requires running the agent in the game numerous times to compute a lower bound on the expected return. 
Beyond this, we are also interested in the robustness of the agent in single game runs against adversaries. 
To this end, we extend the adversarial attack methods from previous literature~\cite{lutjens2020certified, kumar2021policy} to evaluate \texttt{CAMP}.

Given a trained policy $\pi$, the attack perturbs the observed state $s$ by a perturbation $\delta$ to generate a misleading action $a' = \argmax_{a} \pi(s+\delta)_a$.
Compared to the original action $a^* = \argmax_{a} \pi(s)_a$ returned based on the unperturbed observation, $a'$ is a minimizer of the original Q-value $Q_{\pi}(s, a')$.
Specifically, we use Projected Gradient Descent (PGD)~\cite{madry2017towards} and AutoAttack with AutoPGD (APGD)~\cite{croce2020reliable} to generate perturbations constrained by $\ell_2$-norm budgets.
In the attacks, we minimize the Cross-Entropy loss between $\pi(s+\delta)$ and (the one-hot vector of) a target action $a'$ from the action space $\sA$ and iteratively search for the $a'$ that produces the lowest $Q_{\pi}(s, a')$ value across the action space.
The total $\ell_2$ budget is fixed, and any remaining budget from the current time step rolls over to the next.
The attack at each time step terminates if $\argmax_{a}\pi(s+\delta)_a$ equals $a'$ or the perturbation budget is exhausted.
The detailed algorithms are presented in Appendix~\ref{append:algorithms}.

We measure empirical robustness by the average return from $1000$ independent episode plays. 
The average returns of different agents across various environments under different perturbations are illustrated in Figures~\ref{fig:empirical_comparison}-\ref{fig:empirical_comparison_apgd}, with error bars indicating variability.
Notably, agents under attack demonstrate higher performance than the certified results shown in Figures~\ref{fig:cert_comparison} and \ref{fig:full_cert_comparison}.
This discrepancy arises because 1) certified expected returns represent worst-case scenarios and serve as lower bounds for the average returns in Figures~\ref{fig:empirical_comparison}-\ref{fig:empirical_comparison_apgd}, and 2) empirical attacks consume the perturbation budget from the initial time step rather than strategically allocating perturbations to steps causing the most significant reward loss.

\texttt{CAMP} outperforms the baseline in most of the evaluated scenarios with non-zero Gaussian noise augmentations.
In both Figure~\ref{fig:empirical_comparison} and Figure~\ref{fig:empirical_comparison_apgd}, we first assess the Cartpole-1/Cartpole-5 agents trained with noisy observations, where $\sigma\in\{0, 0.2, 0.6\}$, and subjected to perturbations with $\ell_2$ budget caps $\tau$ from $\{0.2, 0.4, 0.6, 0.8, 1.0\}$.
A Gaussian agent with $\sigma=0$ serves as a baseline representing a completely undefended agent, while Gaussian agents with $\sigma\in\{0.2, 0.6\}$ are trained using the Gaussian baseline defense under moderate to high noise levels.
In comparison, \texttt{CAMP} agents with $\sigma \in \{0.2, 0.6\}$ are evaluated against their corresponding Gaussian baselines to highlight improvements in robustness.
In the undefended scenario, agents trained with \texttt{CAMP} still exhibit greater robustness when facing large perturbations.

Similarly, we apply attacks to Gaussian and \texttt{CAMP} agents in Highway, selecting $\sigma \in \{0, 0.8, 1.6\}$ and attack budgets $\tau \in \{0.4, 0.8, 1.2, 1.6, 2.0\}$.
Compared to Gaussian agents, \texttt{CAMP} agents are highly robust against both PGD and APGD attacks.
According to the results, all Gaussian agents experience significant performance degradation under attacks with $\tau \geq 0.4$, whereas \texttt{CAMP} agents trained with $\sigma \in \{0.2, 0.6\}$ remain unaffected. 
Even when $\sigma = 0$, the \texttt{CAMP} agent exhibits superior robustness under attack budgets of $\tau \geq 0.8$ against PGD and $\tau \geq 1.2$ against APGD.

In Atari games, we set $\sigma \in \{12.75, 25.5\}$ and the perturbation budgets are chosen from $\{12.75,38.25,63.75,89.25\}$.
Both $\sigma$ and perturbation budget values are normalized to the range $[0,1]$ in the plot for Atari games.
Agents in Atari games are generally more robust against both attacks.
Augmenting observations with noise enhances agent robustness in both the Gaussian and \texttt{CAMP} frameworks. 
However, the average return of Gaussian agents is more prone to observation noise and declines rapidly with increasing attack budgets in Bank Heist.
In contrast, \texttt{CAMP} effectively mitigates this impact and maintains a higher average return across diverse attack scenarios. 
These results indicate that \texttt{CAMP} not only advances reward certification but also improves the robustness of Gaussian-augmented DRL agents against empirical attacks in individual episodes.

\begin{table}[t]
\caption{Return of Agents Trained in Noise-Free Environments}
\label{table:zero_noise_utility}
\centering
\resizebox{.5\columnwidth}{!}{%
\begin{tabular}{cccc}
\toprule
\multirow{2}{*}{Game}   & \multicolumn{3}{c}{Method} \\
\cmidrule(r){2-4}
                        &\texttt{CAMP}&  Gaussian   &  NoisyNet   \\
\midrule
Cartpole-1              &   200.00    &    199.94   &   200.00    \\
Cartpole-5              &   195.81    &    199.50   &   199.68    \\
Pong1r                  &     0.91    &      0.90   &     -       \\
Freeway                 &     4.00    &      4.00   &     -        \\
\bottomrule
\end{tabular}
}
\end{table}

\noindent \textbf{Impact on training in normal environments.~}
We also evaluated the performance of \texttt{CAMP} when training DRL agents in environments without observation noise.
We test the trained agents and record their test returns in Table~\ref{table:zero_noise_utility}. 
Each return in the table is averaged from 100 independent game runs. 
It can be observed that \texttt{CAMP} does not degrade the training performance in noise-free environments, supporting the versatility of \texttt{CAMP}.

\begin{table}[t]
\caption{Minimal Q-gap under Varying $\lambda$ and $\sigma$ Values}
\label{table:q_gap}
\centering
\resizebox{1.0\linewidth}{!}{%
\begin{tabular}{cccccccc}
\toprule
\multirow{2}{*}{Game}        & \multirow{2}{*}{Method}     & \multicolumn{6}{c}{$\sigma$} \\
\cmidrule(r){3-8}
                             &                             & 0.0 & 0.2 & 0.4 & 0.6 & 0.8 & 1.0 \\
\midrule
\multirow{3}{*}{Cartpole-1}  & Gaussian                    &  0  &  0  &  0  &  0  &  0  &   0  \\
                             & NoisyNet                    &  0  &  0  &  0  &  0  &  0  &   0  \\
                             & \texttt{CAMP}($\lambda=0.5$)& $1.49\times 10^{-8}$ & $6.37\times 10^{-7}$ & $1.25\times 10^{-5}$ & $4.55\times 10^{-6}$ & $2.68\times 10^{-5}$ & $1.30\times 10^{-7}$ \\
                             & \texttt{CAMP}($\lambda=1$)  & $2.91\times 10^{-7}$ & $2.89\times 10^{-6}$ & $9.82\times 10^{-6}$ & $9.57\times 10^{-5}$ & $2.10\times 10^{-5}$ & $4.82\times 10^{-6}$ \\
                             & \texttt{CAMP}($\lambda=4$)  & $1.86\times 10^{-7}$ & $2.98\times 10^{-8}$ & $4.10\times 10^{-7}$ & $2.52\times 10^{-5}$ & $2.46\times 10^{-5}$ & $5.58\times 10^{-9}$ \\
                             & \texttt{CAMP}($\lambda=16$) & $0$                  & $2.24\times 10^{-8}$ & $1.86\times 10^{-9}$ & $0$                  & $9.31\times 10^{-9}$ & $3.73\times 10^{-9}$ \\
\midrule
\multirow{3}{*}{Cartpole-5}  & Gaussian                    &  0   &  0  &  0  &  0  &  0  &  0  \\
                             & NoisyNet                    &  0   &  0  &  0  &  0  &  0  &  0  \\
                             & \texttt{CAMP}($\lambda=0.5$)& $6.71\times 10^{-8}$ & $2.48\times 10^{-6}$ & $3.24\times 10^{-6}$ & $9.24\times 10^{-7}$ & $2.18\times 10^{-6}$ & $1.30\times 10^{-6}$ \\
                             & \texttt{CAMP}($\lambda=1$)  & $3.35\times 10^{-8}$ & $1.19\times 10^{-7}$ & $1.11\times 10^{-5}$ & $3.43\times 10^{-7}$ & $6.12\times 10^{-5}$ & $3.99\times 10^{-5}$ \\
                             & \texttt{CAMP}($\lambda=4$)  & $2.94\times 10^{-7}$ & $1.86\times 10^{-9}$ & $4.84\times 10^{-8}$ & $3.51\times 10^{-5}$ & $2.02\times 10^{-4}$ & $6.33\times 10^{-8}$ \\
                             & \texttt{CAMP}($\lambda=16$) & $7.45\times 10^{-9}$ & $5.59\times 10^{-9}$ & $0$                  & $5.59\times 10^{-9}$ & $1.10\times 10^{-6}$ & $4.47\times 10^{-8}$ \\
\bottomrule
\end{tabular}
}
\end{table}

\begin{table}[t]
\caption{Training Time for Different Methods}
\label{table:cost}
\centering
\resizebox{.5\columnwidth}{!}{%
\begin{tabular}{cccc}
\toprule
\multirow{2}{*}{Game}   & \multicolumn{3}{c}{Time (GPU Hour)} \\
\cmidrule(r){2-4}
                        &\texttt{CAMP}&  Gaussian   &  NoisyNet   \\
\midrule
Cartpole-1              &     0.98    &     0.47    &    0.62     \\
Cartpole-5              &     0.93    &     0.42    &    0.53     \\
Highway                 &    13.78    &     7.00    &    6.98     \\
Pong1r                  &    27.91    &    12.73    &      -      \\
Freeway                 &    29.75    &    14.60    &      -      \\
Bank Heist              &    20.11    &    19.63    &      -      \\
\bottomrule
\end{tabular}
}
\end{table}

\begin{figure*}[t]
     \centering
     \begin{subfigure}
        \centering
        \includegraphics[width=.19\linewidth]{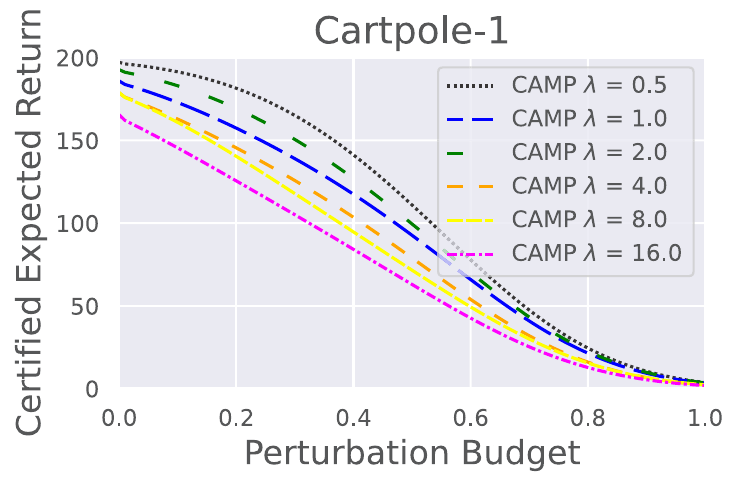}   
     \end{subfigure}
     \begin{subfigure}
        \centering
        \includegraphics[width=.19\linewidth]{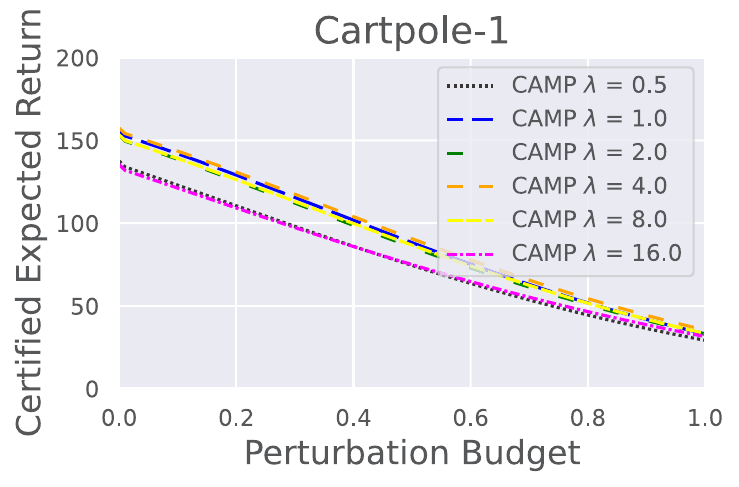}   
     \end{subfigure}
     \begin{subfigure}
        \centering
        \includegraphics[width=.19\linewidth]{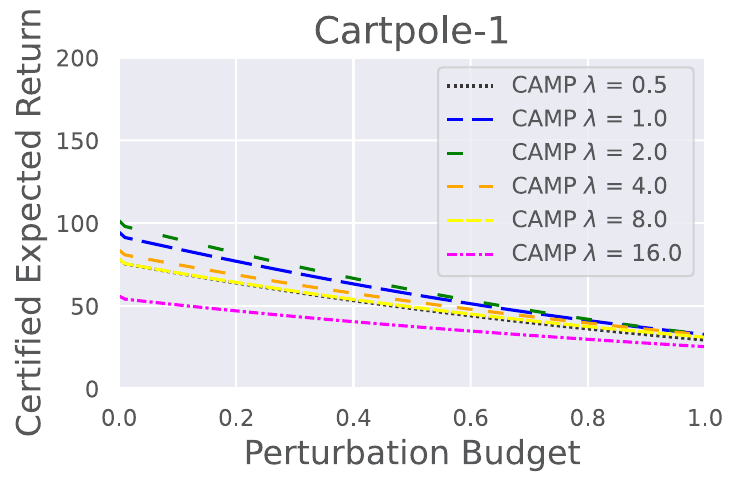}   
     \end{subfigure}
     \begin{subfigure}
        \centering
        \includegraphics[width=.19\linewidth]{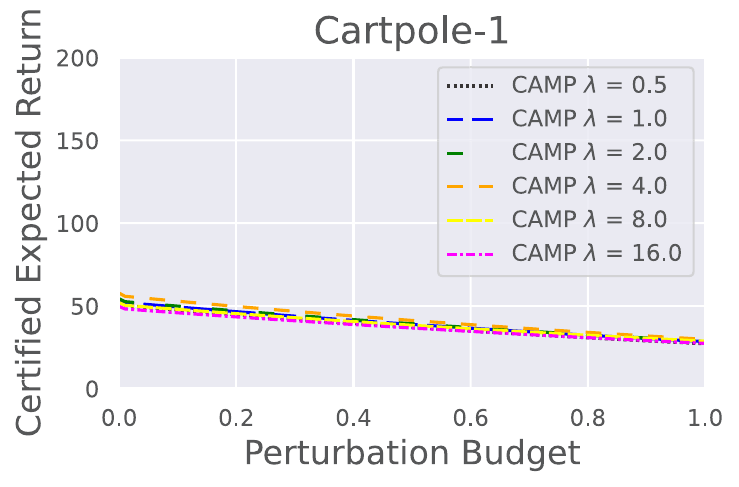}   
     \end{subfigure}
     \begin{subfigure}
        \centering
        \includegraphics[width=.19\linewidth]{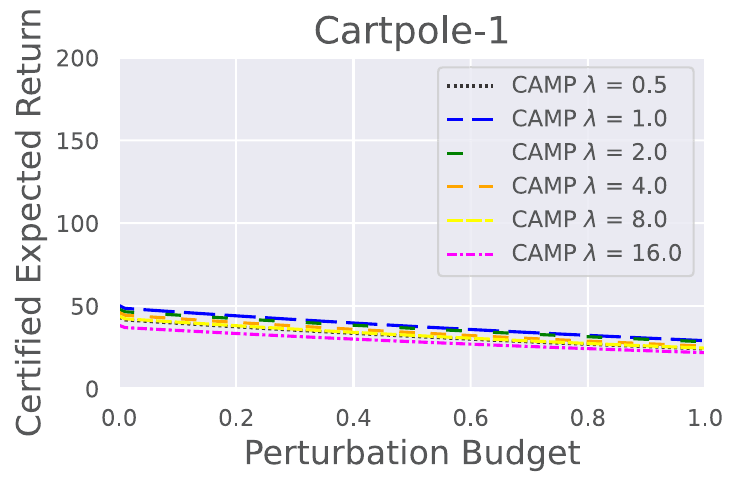}   
     \end{subfigure}
\caption{Ablation on $\lambda$ values in Cartpole-1. The certified expected returns obtained from various $\lambda$ values are shown in the figures. Each figure presents results based on a fixed smoothing noise scale applied to the observed states. From left to right, the smoothing noise scales are 0.2, 0.4, 0.6, 0.8, and 1.0, respectively.}
\label{fig:cartpole_1_ablation}
\end{figure*}

\begin{figure*}[t]
     \centering
     \begin{subfigure}
        \centering
        \includegraphics[width=.19\linewidth]{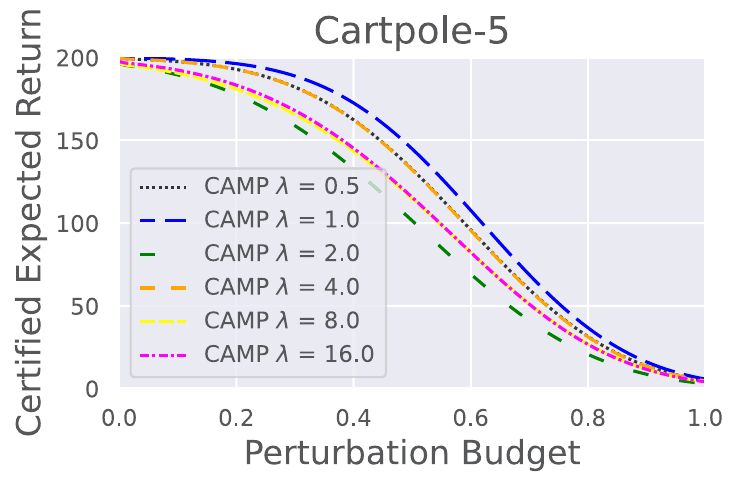}   
     \end{subfigure}
     \begin{subfigure}
        \centering
        \includegraphics[width=.19\linewidth]{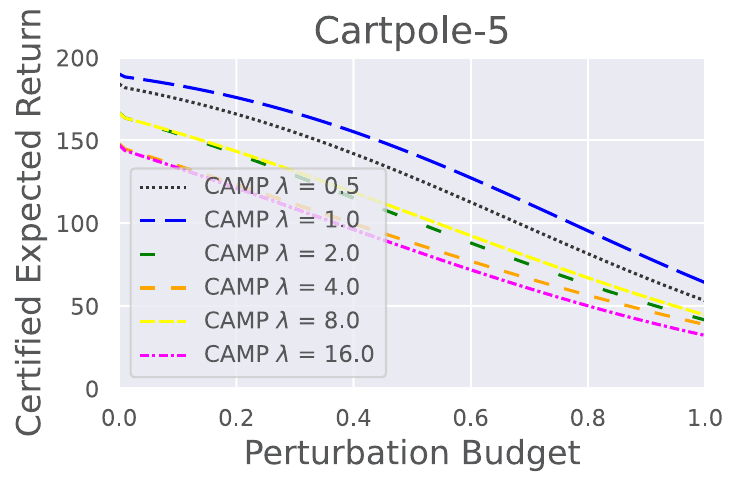}   
     \end{subfigure}
     \begin{subfigure}
        \centering
        \includegraphics[width=.19\linewidth]{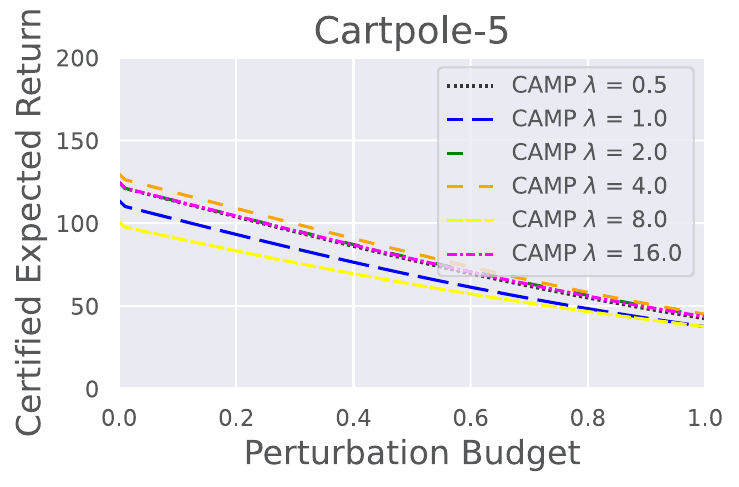}   
     \end{subfigure}
     \begin{subfigure}
        \centering
        \includegraphics[width=.19\linewidth]{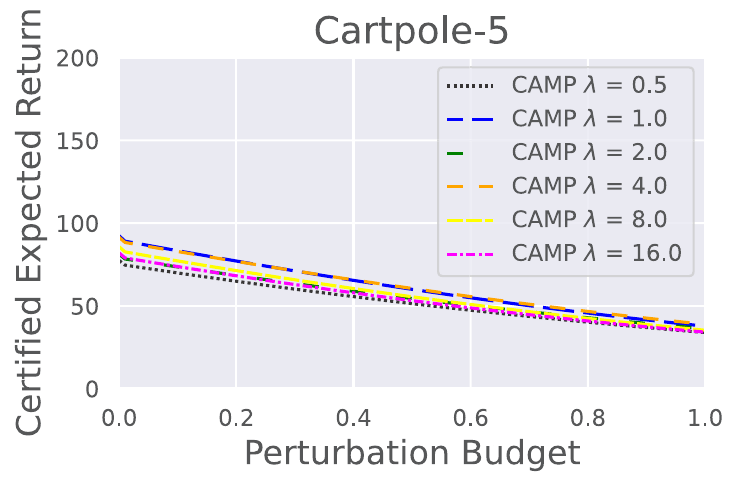}   
     \end{subfigure}
     \begin{subfigure}
        \centering
        \includegraphics[width=.19\linewidth]{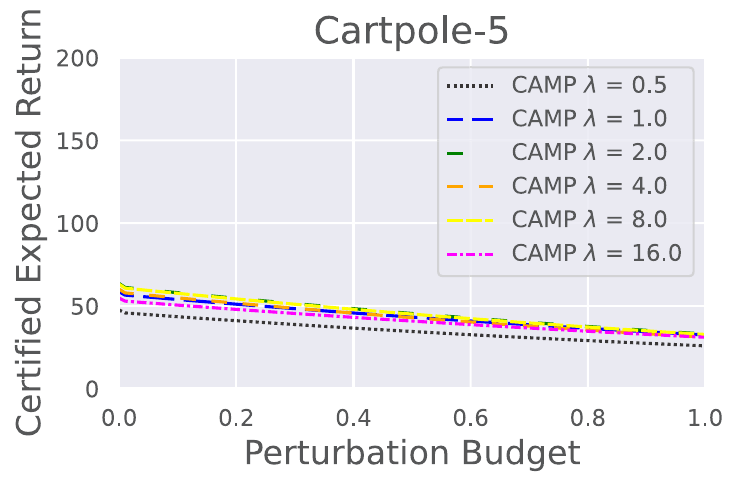}   
     \end{subfigure}
\caption{Ablation on $\lambda$ values in Cartpole-5. The certified expected returns obtained from various $\lambda$ values are shown in the figures. Each figure presents results based on a fixed smoothing noise scale applied to the observed states. From left to right, the smoothing noise scales are 0.2, 0.4, 0.6, 0.8, and 1.0, respectively.}
\label{fig:cartpole_5_ablation}
\end{figure*}

\begin{figure}[t]
     \centering
     \begin{subfigure}
         \centering
         \includegraphics[width=.49\linewidth]{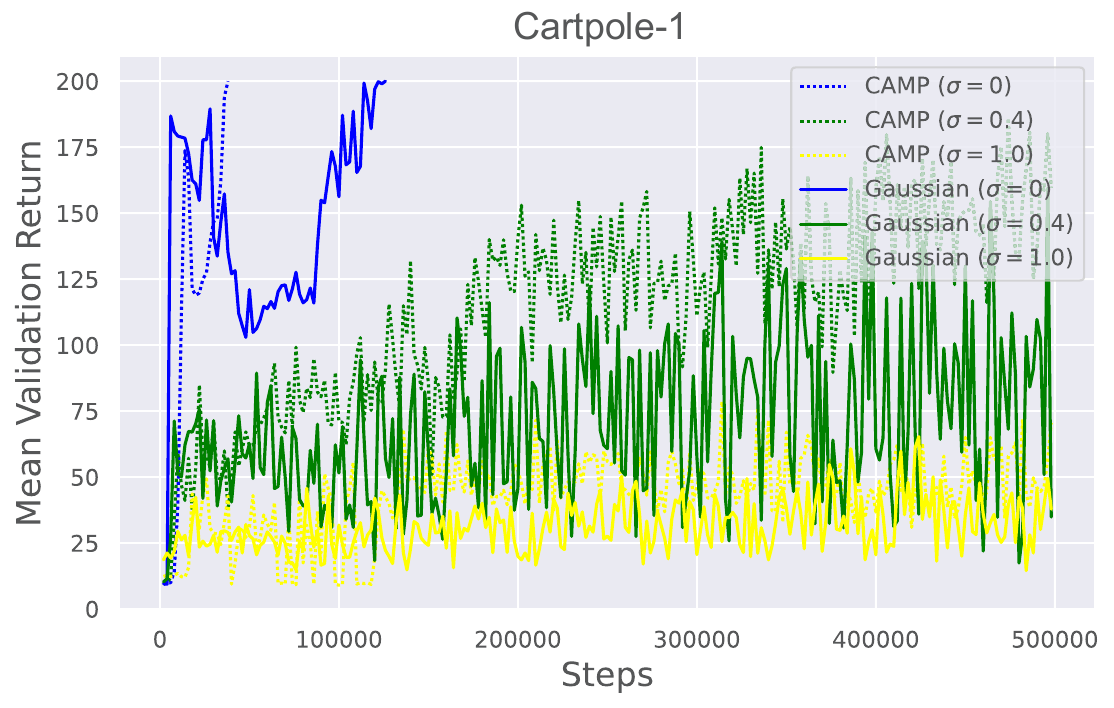}
     \end{subfigure}
     \begin{subfigure}
         \centering
         \includegraphics[width=.49\linewidth]{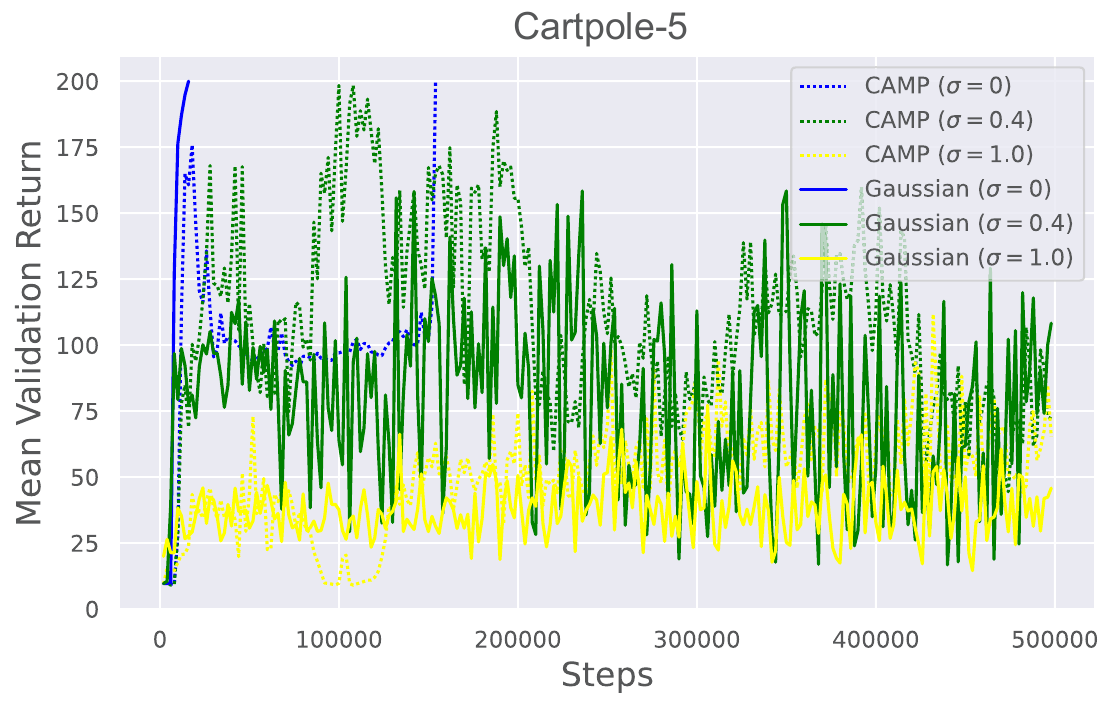}
     \end{subfigure}
     \begin{subfigure}
         \centering
         \includegraphics[width=.49\linewidth]{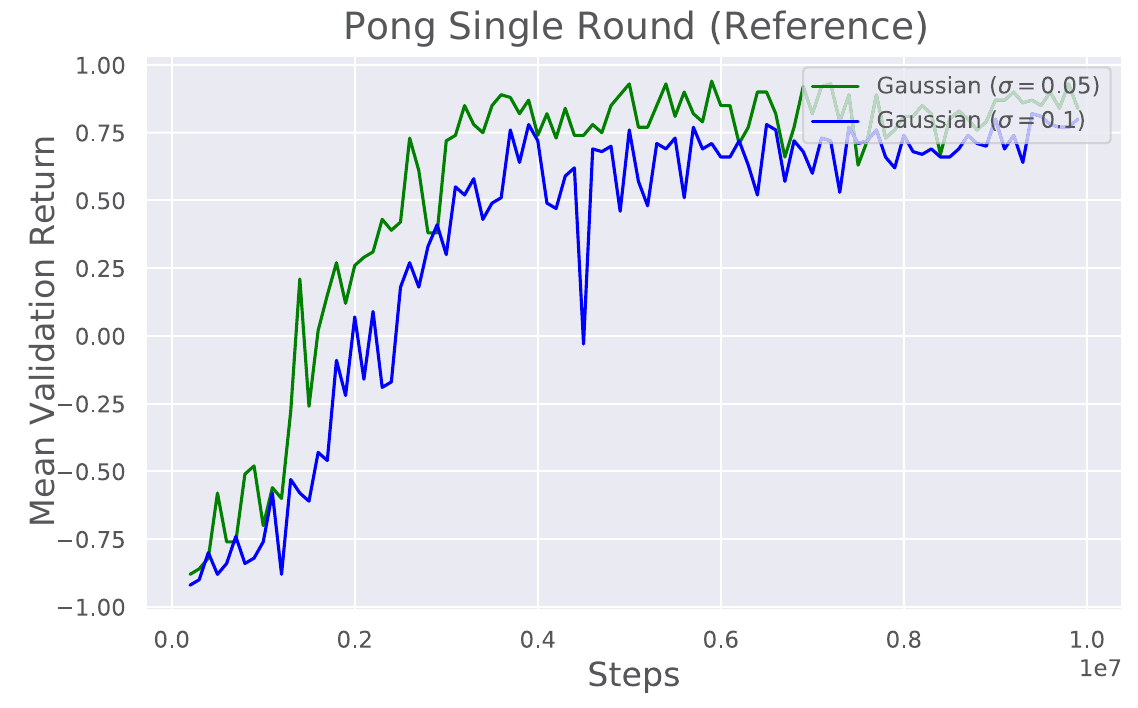}    
     \end{subfigure}
     \begin{subfigure}
         \centering
         \includegraphics[width=.49\linewidth]{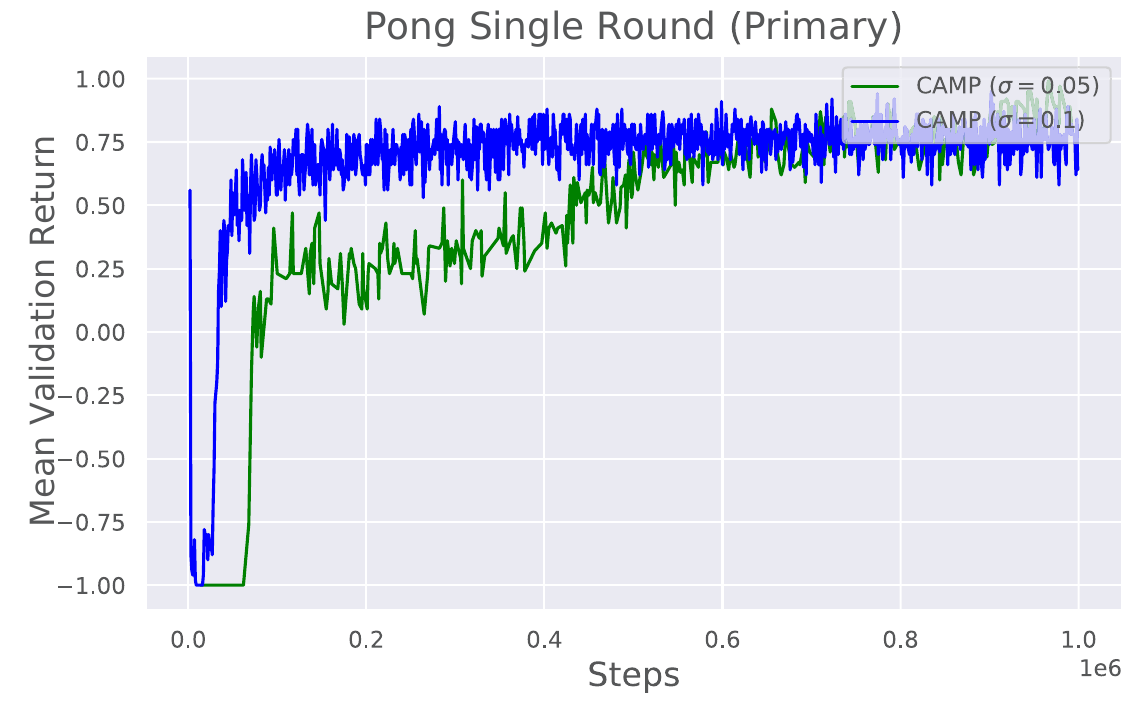}    
     \end{subfigure}
     \begin{subfigure}
         \centering
         \includegraphics[width=.49\linewidth]{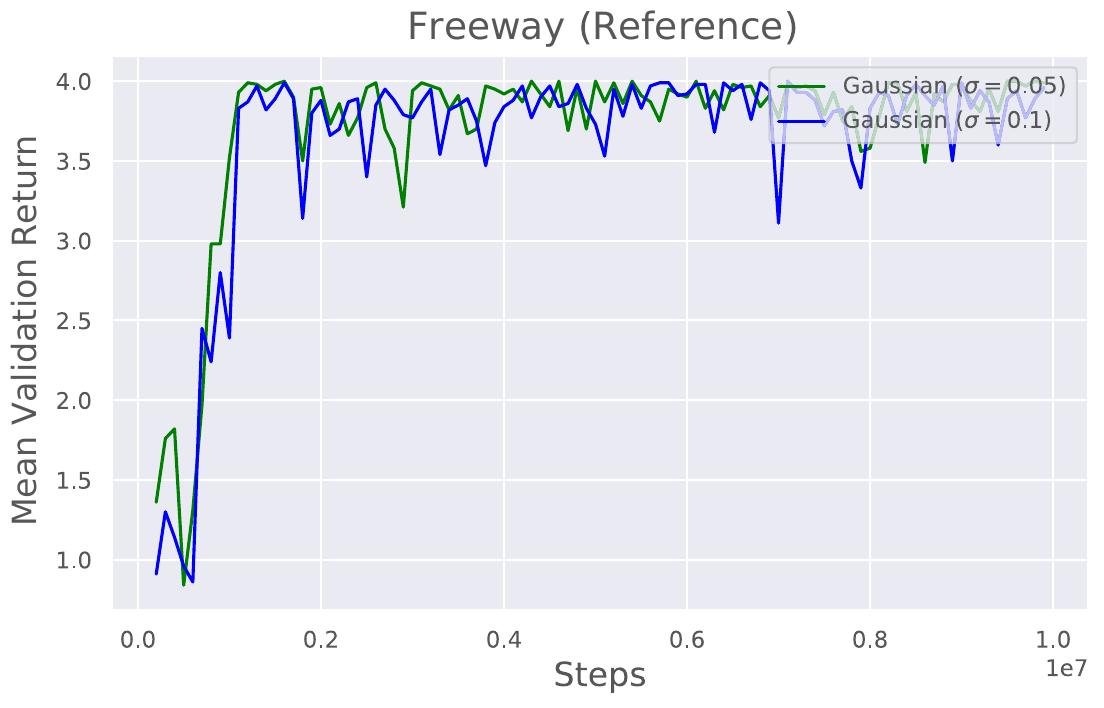}
     \end{subfigure}
     \begin{subfigure}
         \centering
         \includegraphics[width=.49\linewidth]{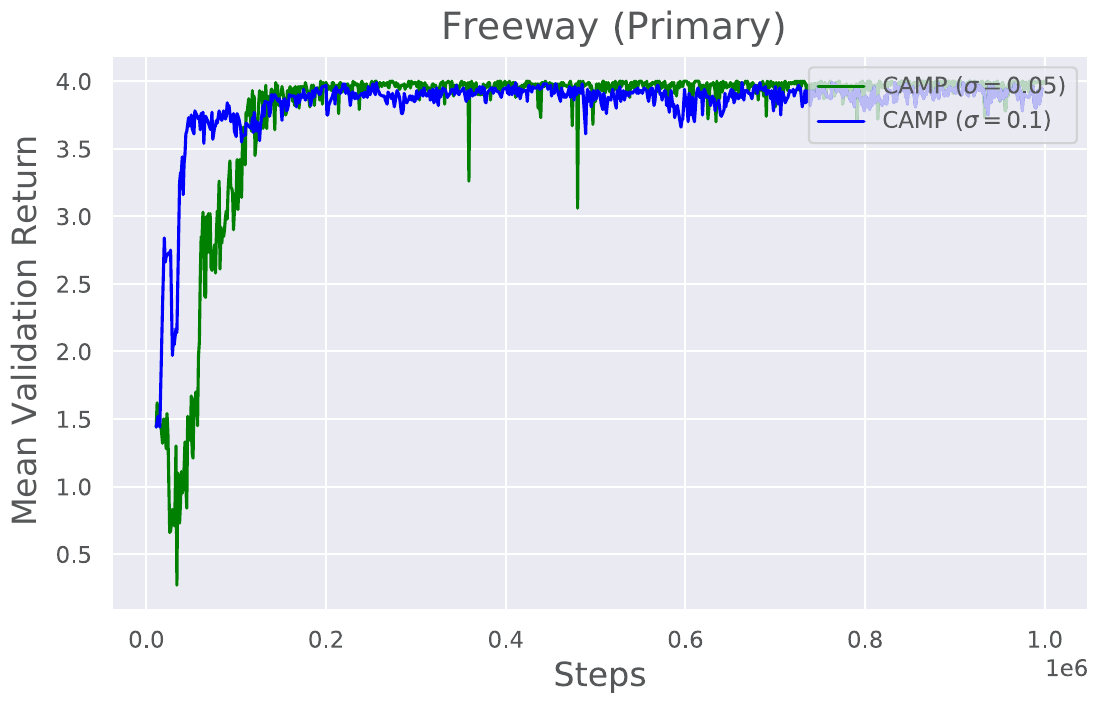}
     \end{subfigure}
\caption{Validation returns in different environments.}
\label{fig:eta_validation_rewards_convergence}
\end{figure}

\subsection{Ablation Studies}\label{subsec:ablation}
In this section, we investigate the impact of $\lambda$ on the trade-off between the certified expected return and certified radius, the effect of \texttt{CAMP} on the policy's Q-gap, and the convergence of the training. 

\noindent \textbf{The effect of $\lambda$.~}
First, we are interested in investigating how the coefficient $\lambda$ impacts the training outcome.
Herein, we train DQN agents on Cartpole-1 and Cartpole-5 using $\lambda\in\{0.5, 1, 2, 4, 8, 16\}$ in the training.
The certification results based on the trained agents are illustrated in Figure~\ref{fig:cartpole_1_ablation} and Figure~\ref{fig:cartpole_5_ablation}.
It can be observed that, generally, the certified performance stays constant over varying $\lambda$ values.
However, the performance drops in a few cases.
Especially, small $\lambda$ for large noise scale or large $\lambda$ for small noise result in the decreases.

\noindent \textbf{Convergence of training .~}
We visualize the mean episode reward during the validation for \texttt{CAMP} and Gaussian, respectively.
The environment noise levels in our visualization are based on $\sigma\in\{0.0, 0.4, 1.0\}$ for Cartpole-1 and Cartpole-5, such that the trends of training in both noise-free and rather noisy environments can be observed.  
The validation rewards are averaged from playing $10$ runs of games after every $2000$ training steps, and the results are plotted in Figure~\ref{fig:eta_validation_rewards_convergence}.
According to the figure, \texttt{CAMP} reaches the maximal return at the same pace with or slightly slower than Gaussian when $\sigma=0$. 
However, when $\sigma$ increases, \texttt{CAMP} can converge to higher mean validation return values in earlier training stages.

On Atari games, the convergence of the reference and primary policies is visualized in separate sub-figures in Figure~\ref{fig:eta_validation_rewards_convergence}, based on Freeway and Pong1r.
The agent with the reference network undergoes validation after every $100k$ training steps, while the primary agent is validated every $1000$ training steps. 
Each validation reward is averaged from $100$ episodes.
The primary policies experience an initial drop in validation return due to the application of the new loss function. 
However, both policies quickly recover, with validation rewards stabilizing before 200k training steps and converging by 1 million steps.

\noindent \textbf{The gap between top-1 and runner-up Q-values.~}
Finally, we investigate the gap between the top-1 action score and the runner-up action score given state observation $s$ under different $\sigma$ settings.
According to Theorem~\ref{theorem:localradius}, larger Q-gaps lead to greater certified local radii and thus make it harder for the adversary at the current step to manipulate the action through observation perturbations.
Particularly, the minimum of the Q-gap is the most crucial statistic signifying the robustness of the agent in worst cases.
Therefore, we would like to examine whether \texttt{CAMP} can effectively widen this minimum of the Q-gap.
In our experiments, we select \texttt{CAMP} agents trained with $\sigma \in \{0.0, 0.2, 0.4, 0.6, 0.8, 1.0\}$ and $\lambda \in \{0.5, 1, 4, 16\}$ to compare with agents trained by Gaussian and NoisyNet.
Each trained agent plays the same $10000$ runs of games and records the minimum of its Q-gap values, as shown in Table~\ref{table:q_gap}.
It can be observed that \texttt{CAMP} effectively upscales the minimal Q-gap and thus makes the agent more robust to noise in the environments. 

\subsection{Computational Cost and Scalability}\label{subsec:runtime}
We document the training costs of \texttt{CAMP} in Table~\ref{table:cost}, where all experiments were conducted using Nvidia H100 GPUs. 
We measured the GPU hours required for training \texttt{CAMP} and compared them to the baseline methods.
The time costs for Cartpole-1 and Cartpole-5 are from the cases where $\sigma=1.0$ and the costs for training in Atari games are based on $\sigma=0.1$.
It is important to note that training typically requires more time to converge under higher levels of observation noise, so the numbers reflect the longest training times for the various methods in each environment.
Note that the time duration includes the computational cost of validation steps.

Although \texttt{CAMP} increases the training duration of DRL agents, the additional time remains within acceptable limits.
First, data parallelism with multiple GPUs can effectively reduce the actual training time.
In particular, when training in environments with high-dimensional observations, \texttt{CAMP} distills its Q-network from a pretrained reference network with a batch size of $32$.
This process can be further accelerated by using larger batch sizes during distillation.
Furthermore, \texttt{CAMP} remains scalable to high-dimensional action spaces. 
Specifically, we use Cross-Entropy for the imitation loss, as it scales efficiently to large category sizes. 
Additionally, the robustness loss relies only on the top-2 Q values, ensuring mathematical scalability to high-dimensional action spaces.

\begin{mdframed}[backgroundcolor=grey!10,rightline=true,leftline=true,topline=true,bottomline=true,roundcorner=1mm,everyline=false,nobreak=false]
\noindent \textbf{Main takeaways.~}
\texttt{CAMP} effectively enhances both the certified expected return and empirical robustness of DQN agents. 
Significant performance gains are observed in classic control and autonomous driving environments (e.g., Cartpole and Highway), while Atari agents, despite showing subtler improvements, are inherently more robust against empirical attacks. 
Further improvements in Atari games might be achieved by training the primary Q-network from reference network checkpoints instead of distilling from a fully trained reference Q-network.
Additionally, \texttt{CAMP} is resilient to hyper-parameter variations and scales effectively to large action spaces.
\end{mdframed}

\section{Related Work}
\noindent\textbf{Adversarial attack and robust DRL.~}
DRL has been demonstrated to be vulnerable to adversarial attacks. 
Early attacks adapted adversarial strategies from neural network classifiers to perturb the observations of DRL agents~\cite{behzadan2017vulnerability,huang2017adversarial,kos2018adversarial}.
Later, attacks were developed to degrade agent performance through interactive adversarial games~\cite{pinto2017robust,pattanaik2018robust,gleave2020adversarial}.
Other attacks have exploited vulnerabilities across different steps~\cite{sun2020stealthy} and used policy-independent perturbations~\cite{korkmaz2023adversarial}.
Beyond targeting a single DRL agent, adversaries have also sought to compromise multi-agent DRL systems~\cite{lin2020robustness}.
The robustness of DRL has garnered substantial attention in recent years. 
Building on early $H_{\infty}$ robust control theory for worst-case performance under deterministic dynamics~\cite{bacsar2008h}, methods have emerged that apply robust control to Markov Decision Processes (MDPs) by modeling uncertainty in transition matrices~\cite{nilim2003robustness,roy2017reinforcement,ho2018fast}.
Additionally, policy smoothing and observation smoothing have been advocated to improve the robustness of DRL~\cite{shen2020deep,sinha2022s4rl,yang2022rorl}.
More recently, some approaches have addressed adversarial robustness in offline DRL through action randomization~\cite{pinto2017robust,kamalaruban2020robust}, and several works have incorporated adversarial training as a defense mechanism~\cite{zhang2020robustwoa,zhang2020robust,liu2023robustness}.
Furthermore, there are defenses against adversarial policies in multi-agent DRL~\cite{guo2023patrol,liu2024rethinking}.
Importantly, significant efforts have been made towards providing theoretical robustness guarantees for DRL~\cite{zhou2021finite,yang2022toward,li2022policy}.
However, these approaches often rely on specific assumptions and are not easily scalable to high-dimensional DRL problems.

\noindent\textbf{Randomized smoothing and certifiably robust DRL.~}
RS has emerged as a practical tool for providing certified robustness, building on previous works that relied on differential privacy and R\'enyi divergence~\cite{lecuyer2019certified,li2019certified,cohen2019certified}.
Numerous efforts have been made to certify $\ell_p$ robustness radii for classifiers, using approaches such as the Neyman-Pearson lemma and beyond~\cite{cohen2019certified,fischer2020certified,li2021tss, hao2022gsmooth,sukenik2022intriguing,cullen2022double,dvijotham2020framework,zhang2020black,salman2020denoised}.
RS has also been extended to certify the learnability of unlearnable data~\cite{wang2025provably}.
Additionally, the limitations of RS against $\ell_p$ ($p>2$) adversaries in high-dimensional input spaces have been identified and are being addressed~\cite{kumar20bcurse,yang2020randomized,li2022double,shu2024effects}.
In the context of reinforcement learning, three closely related works certify a lower bound on the returns of DRL agents through randomized policies~\cite{kumar2021policy,wu2022crop,mu2024reward}.
Several studies have also explored the certified robustness of multi-agent DRL~\cite{sun2022certifiably}.
A recent work applies denoising smoothing to improve the utility of provably robust DRL agents~\cite{sun2024break}.
However, the development of methods for obtaining policies specifically tailored for RS has been overlooked in the current literature.

\section{Conclusion}
We propose \texttt{CAMP}, a method designed to train DRL policies with enhanced certified utility and robustness for policy-smoothing-based robustness certification. 
This approach is grounded in the analysis of certified lower bounds of expected returns obtained by smoothed policies. 
We reformulate the certification problem by decomposing it into two sub-problems which are maximizing expected return and maximizing certified radius. 
While maximizing the expected return is straightforward, maximizing the certified radius presents challenges due to the lack of a differentiable, closed-form expression for the radius. 
We address this challenge by converting the certified radius into a soft radius and further transforming it into a per-step radius, which characterizes the maximum perturbation that does not degrade the current optimal expected return.
Through experiments, we have verified the effectiveness of \texttt{CAMP} in improving certified expected returns at fixed certified radii in environments with both simple and high-dimensional observations. 
Additionally, \texttt{CAMP} agents demonstrate superior robustness against empirical adversarial perturbations in each game run.
Nevertheless, there are limitations to be addressed.
\texttt{CAMP} currently supports only discrete action spaces.
Additionally, it selects the top-1 and runner-up Q-values based on a presumed optimal policy, which may not be fully captured by the reference network.
Moreover, \texttt{CAMP} exhibits less pronounced improvements in environments with high-dimensional visual observations compared to control and autonomous driving scenarios.
To address these limitations, future work will extend \texttt{CAMP} to continuous action spaces, refine the training process, and reduce overhead to improve accessibility.

\section*{Acknowledgments}
This work has been supported by the Cyber Security Research Centre Limited whose activities are partially funded by the Australian Government's Cooperative Research Centres Programme.

\section*{Ethics Considerations}
We attest that the research topic, research process, and possible future publication of the research content are ethical.
This research aims to facilitate the provable robustness of deep reinforcement learning agents by proposing a novel training method.
To the best of our knowledge, the proposed method neither introduces extra risks nor is exploitable for potential adversaries.
Moreover, the entire research process involves no experiments with live systems or human/animal participants.

\section*{Open Science}
This paper adheres to the principles of open science by ensuring that all research materials, data, and code used in the study are publicly accessible and fully documented. 
The experimental environments are available in open libraries, promoting transparency and reproducibility of the results. 
The methodology is described in detail, with the complete code and scripts permanently hosted on \href{https://zenodo.org/records/14729675}{https://zenodo.org/records/14729675} and Github to support replication and validation by other researchers.
Furthermore, the study's findings are shared under an open-access license, ensuring that the scientific community and the public can freely access, use, and build upon the work.
This commitment to open science promotes collaboration, enhances the credibility of the research, and contributes to the collective advancement of knowledge in the field.

\bibliographystyle{plain}
\bibliography{advref_abbrev}


\appendix

\begin{figure*}[t]
     \centering
     \begin{subfigure}
         \centering
         \includegraphics[width=.32\linewidth]{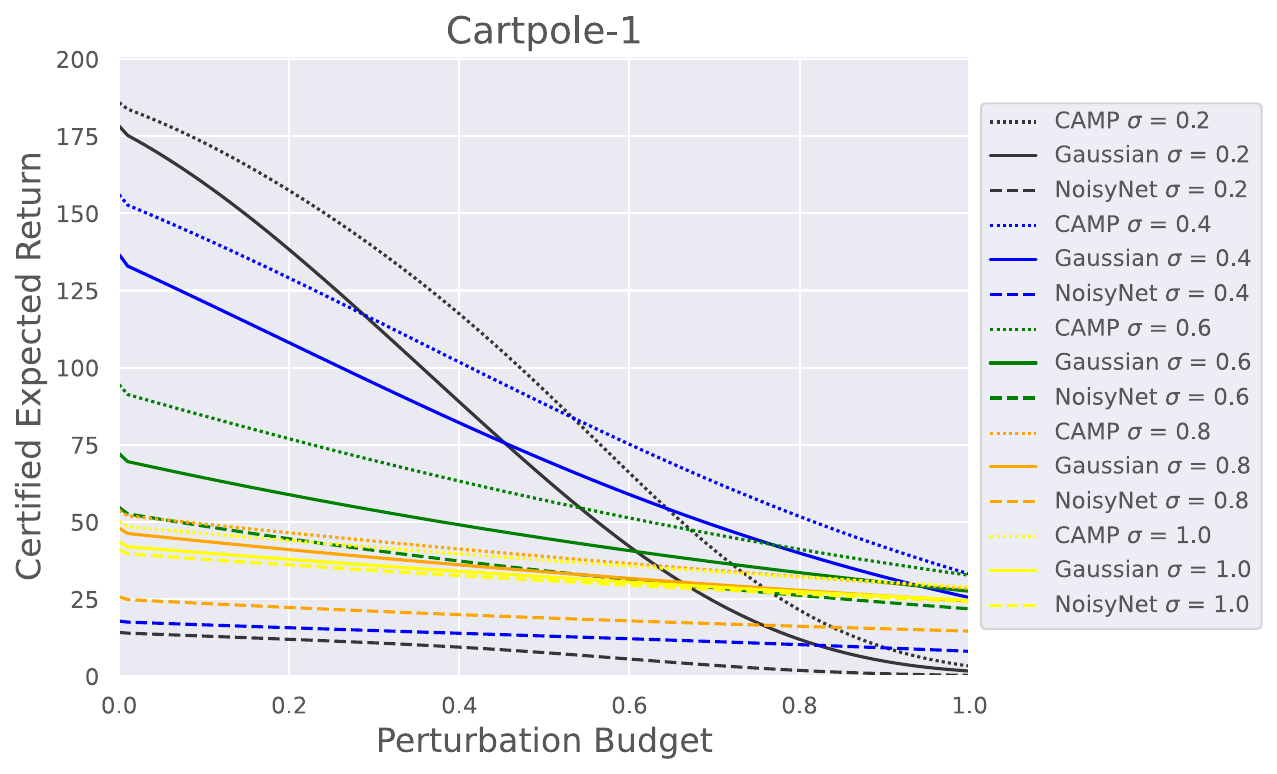}    
     \end{subfigure}
     \begin{subfigure}
         \centering
         \includegraphics[width=.32\linewidth]{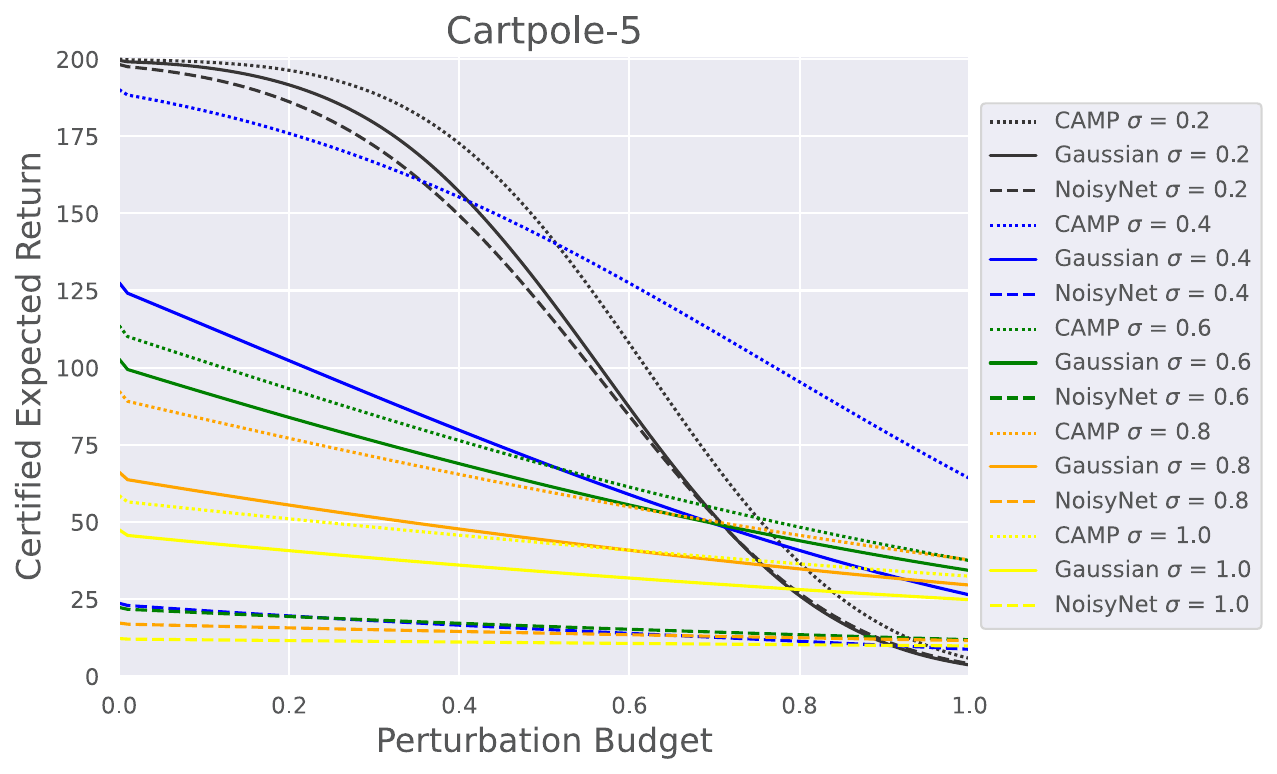}
     \end{subfigure}
     \begin{subfigure}
         \centering
         \includegraphics[width=.32\linewidth]{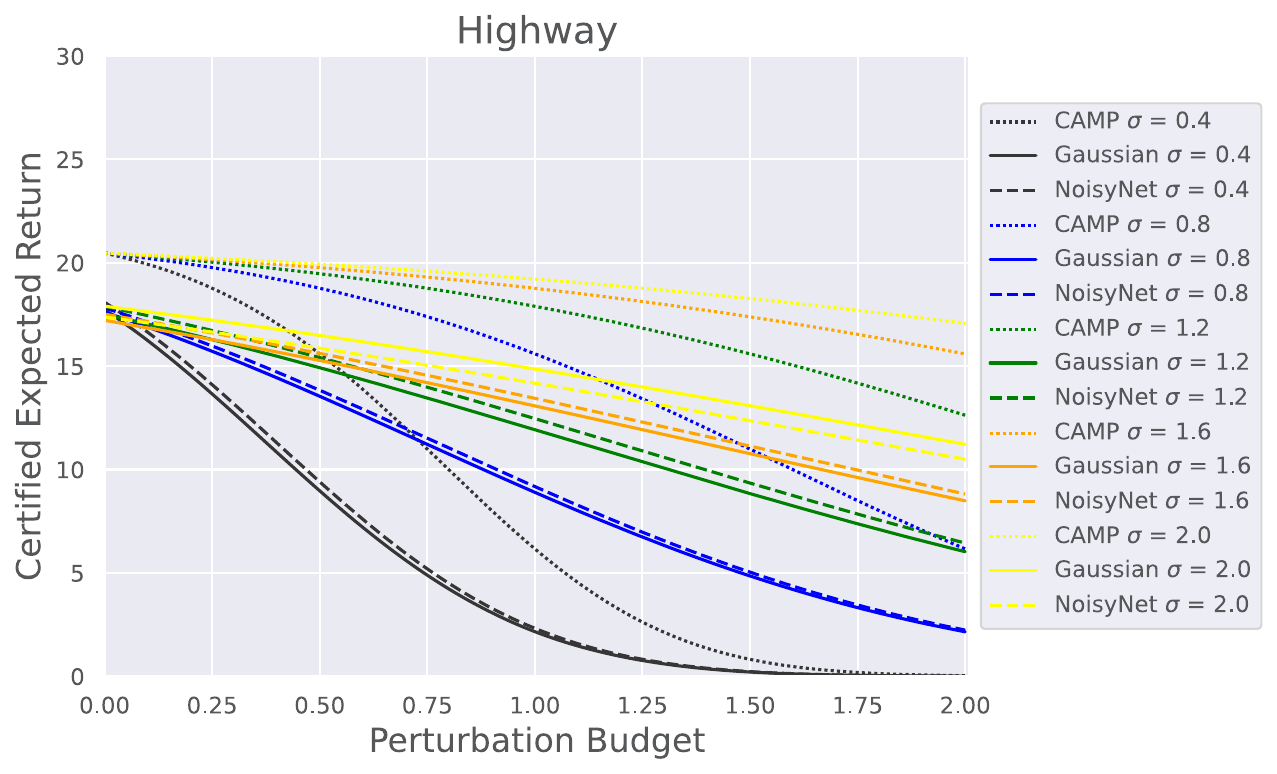}
     \end{subfigure}
\caption{Full certification results on CartPole-1, Cartpole-5, and Highway.}
\label{fig:full_cert_comparison}
\end{figure*}

\section{Proofs}\label{append:proof}
\hardradius*
\begin{proof}
By the lower bound on $\E[F_{\pi}(\rvz)]$ and algebra,
\footnotesize
\begin{equation}
    \begin{aligned}
        \E[F_{\pi}(\rvz')] \geq & \rvr_1 \cdot \Phi \left(\Phi^{-1} \left( \underline{P_{\pi}^{\rvz}}(\rvr_1) \right)-\frac{\tau}{\sigma} \right) \\
        + & \sum_{i=2}^{m}(\rvr_i - \rvr_{i-1}) \cdot \Phi \left(\Phi^{-1} \left( \underline{P_{\pi}^{\rvz}}(\rvr_i) \right)-\frac{\tau}{\sigma} \right), \\
        \geq & \rvr_1 \cdot \Phi \left(\Phi^{-1} \left( \underline{P_{\pi}^{\rvz}}(\rvr_1) \right)-\frac{\tau}{\sigma} \right) \\
        \forall\, & \|\Delta\|_2 \leq \tau.
    \end{aligned}
\end{equation}
\normalsize
Therefore to make $\E[F_{\pi}(\rvz')] \geq \xi$, it is sufficient if
\footnotesize
\begin{equation}
    \begin{aligned}
        & \rvr_1 \cdot \Phi \left(\Phi^{-1} \left( \underline{P_{\pi}^{\rvz}}(\rvr_1) \right)-\frac{\tau}{\sigma} \right) \geq \xi, \\
    \iff& \tau \leq \sigma \left[ \Phi^{-1} \left( \underline{P_{\pi}^{\rvz}}(\rvr_1) \right) -  \Phi^{-1} \left( \xi/\rvr_1 \right) \right].
    \end{aligned}
\end{equation}
\normalsize
\end{proof}

\cdftoexpectation*
\begin{proof}
Since $\Psi_{X}(C)$ and $\Psi_{Y}(C)$ are CDFs, by definition:
\footnotesize
\begin{equation}
\begin{aligned}
     \Psi_{X}(C) & = 1 - \Pr[X \geq C] \\
                 & = 1 - \int_{C}^{+\infty} f_{X}(t) \mathrm{d}t.   
\end{aligned}
\end{equation}
\normalsize
Note that given $X \geq 0$ and $ Y \geq 0 $,
\footnotesize
\begin{equation}
    \begin{aligned}
        \int_{0}^{+\infty} [1 - \Psi_{X}(C)] \mathrm{d} C & = \int_{0}^{+\infty} \Pr[X \geq C] \mathrm{d}C \\
        & = \int_{0}^{+\infty} \int_{x}^{+\infty} f_{X}(t) \mathrm{d}t \mathrm{d}x \\
        & = \int_{0}^{+\infty} \int_{0}^{t} f_{X}(t) \mathrm{d}x \mathrm{d}t  \\
        & = \int_{0}^{+\infty} t f_{X}(t) \mathrm{d}t.
    \end{aligned}
\end{equation}
\normalsize
Let $t = C$ and thus $\mathrm{d}t = \mathrm{d}C$, there is
\footnotesize
\begin{equation}
    \begin{aligned}
        \int_{0}^{+\infty} (1 - \Psi_{X}(C)) \mathrm{d} C & = \int_{0}^{+\infty} t f_{X}(t) \mathrm{d}t \\
        & = \int_{0}^{+\infty} C f_{X}(C) \mathrm{d}C\\
        & = \E [X]. 
    \end{aligned}
\end{equation}
\normalsize
In the same spirit, we can obtain
\footnotesize
\begin{equation}
    \E[Y] = \int_{0}^{+\infty}[1 - \Psi_{Y}(C)] \mathrm{d} C.
\end{equation}
\normalsize
Therefore we have
\footnotesize
\begin{equation}
    \begin{aligned}
         & \E[X] \leq \E[Y] \\
    \iff & \int_{0}^{+\infty} [1 - \Psi_{X}(C)] \mathrm{d} C \leq \int_{0}^{+\infty}[1 - \Psi_{Y}(C)] \mathrm{d} C\\
    \iff & \int_{0}^{+\infty} [\Psi_{Y}(C) - \Psi_{X}(C)] \mathrm{d} C \geq 0.
    \end{aligned}
\end{equation}
\normalsize
\end{proof}

\lipschitz*
\begin{proof}
    See proofs of Lemma 2 in Fan \textit{et al.}~\cite{wu2022crop}.
\end{proof}

\softradius*
Refer to the proof in the Section~\ref{subsec:radius_max}.

\localradius*
\begin{proof}
According to the Lemma 1 of Salman et al.~\cite{salman2019provably}, $\E_{\epsilon_t}Q_{\pi}(s_t+\epsilon_t, a_t)$ is $\sqrt{2/\pi}\frac{u_i - l_i}{\sigma}$-Lipschitz and $\Phi^{-1}\left( \frac{\E_{\epsilon_t}Q_{\pi}(s_t+\epsilon_t, a_t) -l_i }{u_i - l_i} \right)$ is 1-Lipschitz.
Since $\bar{Q}_{\pi^*}(s_t, a_t) = \max_{\pi}\E_{\epsilon_t}Q_{\pi}(s_t+\epsilon_t, a_t)$, without loss of generality, $\bar{Q}_{\pi^*}(s_t, a_t)$ has the same Lipschitz continuity.
Let $a^{(1)}_i = \argmax_{a_i} \bar{Q}_{\pi^*}(s_i, a_i)$ and $a^{(2)}_i = \argmax_{a_i:a_i\neq a^{(1)}_i} \bar{Q}_{\pi}(s_i, a_i)$.
Based on Lemma~\ref{lemma:lipschitz}, there are
\footnotesize
\begin{equation}\label{eq:perturb_gap1}
    \begin{aligned}
        \Phi^{-1}\left( \frac{\bar{Q}_{\pi^*}(s_i, a^{(1)}_i) -l_i }{u_i - l_i} \right) - 
        \Phi^{-1}\left( \frac{\bar{Q}_{\pi^*}(s_i+\delta_i, a^{(1)}_i) -l_i }{u_i - l_i} \right)
        \leq \|\delta_i\|_2,
    \end{aligned}
\end{equation}
\normalsize
and for all $a'_i: a'_i\neq a^{(1)}_i$,
\footnotesize
\begin{equation}\label{eq:perturb_gap2}
    \begin{aligned}
        \Phi^{-1}\left( \frac{\bar{Q}_{\pi^*}(s_i+\delta_i, a^{'}_i) -l_i }{u_i - l_i} \right) -
        \Phi^{-1}\left( \frac{\bar{Q}_{\pi^*}(s_i, a^{'}_i) -l_i }{u_i - l_i} \right)
        \leq \|\delta_i\|_2.
    \end{aligned}
\end{equation}
\normalsize
If the perturbation $\delta_i$ satisfies
\footnotesize
\begin{equation}\label{cond:within_radius}
    \begin{aligned}
        \|\delta_i\|_2 \leq \frac{\sigma}{2} \left[  \Phi^{-1}\left( \frac{\bar{Q}_{\pi^*}(s_i, a^{(1)}_i) - l_i }{u_i - l_i} \right) \right.
        - \left. \Phi^{-1}\left( \frac{\bar{Q}_{\pi^*}(s_i, a^{'}_i)  - l_i }{u_i - l_i} \right)  \right],
    \end{aligned}
\end{equation}
\normalsize
we have:
\footnotesize
\begin{equation}
    \begin{aligned}
        &\Phi^{-1}\left( \frac{\bar{Q}_{\pi^*}(s_i+\delta_i, a^{(1)}_i) -l_i }{u_i - l_i} \right) \\
        \overset{(a)}{\geq} & \Phi^{-1}\left( \frac{\bar{Q}_{\pi^*}(s_i, a^{(1)}_i) -l_i }{u_i - l_i} \right) \\
        & - \frac{1}{2}\left[  \Phi^{-1}\left( \frac{\bar{Q}_{\pi^*}(s_i, a^{(1)}_i) -l_i }{u_i - l_i} \right) \right.\\
        & - \left. \Phi^{-1}\left( \frac{\bar{Q}_{\pi^*}(s_i, a^{'}_i) -l_i }{u_i - l_i} \right)  \right]
    \end{aligned}
\end{equation}
\normalsize
and
\footnotesize
\begin{equation}
    \begin{aligned}
        &\Phi^{-1}\left( \frac{\bar{Q}_{\pi^*}(s_i+\delta_i, a^{'}_i) -l_i }{u_i - l_i} \right) \\
        \overset{(b)}{\leq} & \Phi^{-1}\left( \frac{\bar{Q}_{\pi^*}(s_i, a^{'}_i) -l_i }{u_i - l_i} \right) \\
        & + \frac{1}{2}\left[  \Phi^{-1}\left( \frac{\bar{Q}_{\pi^*}(s_i, a^{(1)}_i) -l_i }{u_i - l_i} \right) \right.\\
        & - \left. \Phi^{-1}\left( \frac{\bar{Q}_{\pi^*}(s_i+\epsilon_i, a^{(2)}_i) -l_i }{u_i - l_i} \right)  \right] \\
        \overset{(c)}{\leq} & \Phi^{-1}\left( \frac{\bar{Q}_{\pi^*}(s_i, a^{(2)}_i) -l_i }{u_i - l_i} \right) \\
        & + \frac{1}{2}\left[  \Phi^{-1}\left( \frac{\bar{Q}_{\pi^*}(s_i, a^{(1)}_i) -l_i }{u_i - l_i} \right) \right.\\
        & - \left. \Phi^{-1}\left( \frac{\bar{Q}_{\pi^*}(s_i, a^{(2)}_i) -l_i }{u_i - l_i} \right)  \right]
    \end{aligned}
\end{equation}
\normalsize
Herein, $(a)$ and $(b)$ hold according to Equation~\ref{eq:perturb_gap1}, Equation~\ref{eq:perturb_gap2}, and Condition~\ref{cond:within_radius}. 
$(c)$ is due to $\bar{Q}_{\pi^*}(s_i, a^{'}_i) \leq \bar{Q}_{\pi^*}(s_i, a^{(2)}_i) $ and the monotonicity of $\Phi^{-1}(\cdot)$.

Therefore, we have
\footnotesize
\begin{equation}
    \begin{aligned}
        & \Phi^{-1}\left( \frac{\bar{Q}_{\pi^*}(s_i+\delta_i, a^{(1)}_i) -l_i }{u_i - l_i} \right)
   \geq \Phi^{-1}\left( \frac{\bar{Q}_{\pi^*}(s_i+\delta_i, a^{'}_i) -l_i }{u_i - l_i} \right) \\
   \iff & \bar{Q}_{\pi^*}(s_i+\delta_i, a^{(1)}_i) \geq \bar{Q}_{\pi^*}(s_i+\delta_i, a^{'}_i),\, \forall a^{'}_i\neq a^{(1)}_i.
    \end{aligned}
\end{equation}
\normalsize
This means that the top-1 action recommended by the policy retains the highest expected return at step $t$.
If each perturbation in $\{\delta_i\}_{i=t}^{T-1}$ follows the above condition, the optimal expected return will not be decreased.
\end{proof}

\section{Experiment Settings}\label{append:exp}

\subsection{Hyper-Parameters and Settings}\label{appendsub:param_setting}
\noindent\textbf{Cartpole.~}
For training on Cartpole-1 and Cartpole-5 using \texttt{CAMP}, we employ a buffer size of 100k, a total of 500k training steps, a batch size of 1024, and a learning rate of $5 \times 10^{-5}$. 
Training begins after 1000 burn-in steps. 
The training frequency is set to 256, with each step involving 128 gradient updates. 
The target Q-network is updated every 10 steps with a Polyak averaging rate of 1. 
An $\epsilon$-greedy policy is used during training, where $\epsilon$ follows a linear schedule starting from 1 and decreasing to 0 over the first 16\% of training steps. 
Validation occurs every 2000 training steps, with the validation return averaged over 10 validation episodes.

\noindent\textbf{Highway.~}
A buffer size of 15k is used for Highway. 
The training lasts for 1000k steps, including 1k burn-in steps.
$\epsilon$ follows a linear schedule starting from 1 and decreasing to 0.05 over the first 10\% of training steps.
The other hyper-parameters are the same as those in the Cartpole environments.

\noindent\textbf{Atari.~}
The reference network training involves a buffer size of 10k, a total of 10 million training steps, a batch size of 32, and a learning rate of $1 \times 10^{-4}$.
Training begins after a burn-in period of 100k steps. 
The training frequency is set to 4, with each training step involving a single gradient step. 
The target Q-network is updated every 1000 steps with a Polyak rate of 1. 
The $\epsilon$-greedy policy's $\epsilon$ value follows a linear schedule, starting at 1 and gradually decreasing to 0.01 over the first 10\% of the training steps. 
Validation occurs every 100k training steps, with the validation return averaged over 100 episodes. 
To prevent excessively long gameplay, the frame number is capped at 250 for all games.
The primary Q-networks are trained for 1 million steps using a 0.01-greedy policy, with validation performed every 1,000 training steps. 
All other settings remain consistent with those used for training reference policies.

We also observed a discrepancy between the certification and empirical robustness results of our Gaussian baseline and those reported in the PS paper. 
The discrepancy may stem from differences in training settings. 
First, our method is implemented using the Clean-RL library~\cite{huang2022cleanrl}, whereas the other implementation is based on Stable Baselines 3. 
Second, during Atari game training, we normalize the observed pixel values to the range $[0,1]$, while the other implementation uses observations in the range $[0,255]$.

\begin{table}[t]
\caption{Network Architectures}
\label{table:arc}
\centering
\resizebox{.65\columnwidth}{!}{%
\begin{tabular}{cc}
\toprule
        MLP        &           Nature CNN \\
\midrule
      $input\ dim$   &          $84 \times 84$         \\
      Linear $input\ dim \times 256$    &       $Conv2d(4, 32, 8, stride=4)$ \\
            ReLU    &       ReLU  \\
      Linear $256 \times 256$   &       $Conv2d(32, 64, 4, stride=2)$ \\
            ReLU    &       ReLU  \\
 Linear $256 \times action\ dim$  &       $Conv2d(64, 64, 3, stride=1)$ \\
        -           &       ReLU \\      
        -           &       Flatten() \\      
        -           &       Linear $3136 \times 512$ \\  
        -           &       ReLU \\           
        -           &       Linear $512 \times action\ dim$ \\           
\bottomrule
\end{tabular}
}
\end{table}

\begin{algorithm}[t]
\footnotesize
\caption{Empirical Attack}\label{alg:emp_attack}
\KwIn{Budget $\tau$, agent with Q-network $\pi$, attack step size $\alpha$, multiplier $\beta$, episode length $T$, Q value gap $q$, environment $\tE$ with state transition function $\gT$ and reward function $\gR$.}
Initialize currently available budget $c = \tau$\\
Initialize $s_0$\\
Initialize $\Delta = \emptyset$\\
\For{$t \in \{0, 1, 2,..., T-1\}$}
{   
    $a_t^* \gets\ \argmax_{a} \pi(s_t)_a$\\
    $Q_t^* \gets\ \pi(s_t)_{a_t^*}$\\
    ${s'}^*_t \gets\ s_t$\\
    $step\_count \gets\ \floor{\beta * c / \alpha}$\\
    \For{$a'\in\sA, a'\neq a_t^* and |a' - a_t| \geq q$}
    {   
        ${s}'_t \gets\ s_t$\\
        \For{$i\in\{1, 2, ..., step\_count\}$}
        {   
            $a_{predicted} \gets\ \argmax_{a}{\pi({s}'_t)_a}$\\
            \If{$a_{predicted} == a'$}
            {
                \If{$\pi(s_t)_{a_{predicted}} < Q_t^*$}
                {
                    $Q_t^* \gets\ \pi(s_t)_{a_{predicted}}$\\
                    ${s'}^*_t \gets\ {s}'_t$\\
                    \textbf{break}\\
                }
            }
            $l_{CE} \gets\ CE(\pi({s}'_t), One\_Hot(a'))$\\
            ${s}'_t \gets\ Perturb\_Step({s}'_t; c, l_{CE}, \alpha)$\\
        }
    }
    $c \gets\ \sqrt{c^2 - \|{s'}^*_t - s_t\|_2^2}$\\
    \If{$c < 0$}
    {
        \textbf{break}\\
    }
    $\Delta \gets\ \Delta \cup\{{s'}^*_t - s_t\}$\\
    $ a_{t}\ \gets\ \argmax_a \tilde{\pi}({s'}^*_t)_a$ \\
    $s_{t+1},\ d_t, r_{t+1}\ \gets \gT({s}_{t}, a_{t}),\ \gR({s}_{t}, a_{t})$ \\
}
\KwOut{Observation perturbations $\Delta$}
\end{algorithm}

\subsection{Additional Results}\label{appendsub:full_certification_results}
The full certification results in Cartpole-1/5 and Highway across different $\sigma$ setups are presented in Figure~\ref{fig:full_cert_comparison}.
It is evident that \texttt{CAMP} consistently outperforms the baseline methods across all $\sigma$ levels, with the most significant improvement occurring in Cartpole-5 when $\sigma=0.4$.

\subsection{Model Architectures}\label{appendsub:arc}
We list the architectures of the Q-networks used in this paper.
For classic control problems, we use MLP following previous papers. 
Nature CNN is used elsewhere for Atari games.
Table~\ref{table:arc} summarizes the architectures of these two networks.

\section{Additional Algorithmic Details}\label{append:algorithms}
We provide the details of the empirical attacks used to evaluate the robustness of the trained policies in Algorithm~\ref{alg:emp_attack}.
$Perturb\_Step({s}'_t; c, l_{CE}, \alpha)$ represents either a single PGD or APGD step as defined in their original literature(\ie PGD step~\cite{madry2017towards}, or Line 7-Line 16, Algorithm 1 in APGD~\cite{croce2020reliable}).
$l_{CE}$ is the Cross-Entropy loss between predicted logits and the one-hot encoding of the target action. 

We slightly modify the APGD attack to attack DRL agents. 
Specifically, we omit random initialization of adversarial examples since it quickly runs out of perturbation budget and leads to unsuccessful attacks.
Moreover, we decrease the initial step size $\alpha$ to $0.05*\tau$ ($0.5*\tau$ in Bank Heist), such that the budget can be better preserved and allocated across the episode.
In both attacks, the step count at each time step is calculated as $\beta * c / \alpha$, where $c$ is the perturbation budget available at the current step and $\beta=2$.
On the side of PGD, we apply the normalized gradients rather than gradient signs in each attack step.
The step size $\alpha$ in PGD is 0.01.

APGD halves its step size during the attack process, allowing it to meet the attack criteria (i.e., causing observations to be classified into target actions by the Q-network) with smaller perturbations at each step.
As a result, APGD may induce a more significant drop in the average return.
Nevertheless, across all environments, \texttt{CAMP} agents consistently demonstrate more robust performance compared to baseline agents.

\newpage


\section*{Supplementary material}

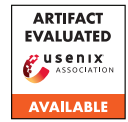 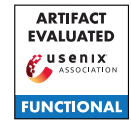 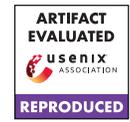

# USENIX Security '25 Artifact Appendix: CAMP in the Odyssey: Provably Robust Reinforcement Learning with Certified Radius Maximization

Derui Wang[♡,♠], Kristen Moore[♡,♠], Diksha Goel[♡,♠], Minjune Kim[♡,♠], Gang Li[♣], Yang Li[♣], Robin Doss[♣], Minhui Xue[♡,♠], Bo Li[◇], Seyit Camtepe[♡,♠], and Liming Zhu[♡]

[♡]*CSIRO's Data61, Australia*
[♠]*Cyber Security Cooperative Research Centre, Australia*
[♣]*Deakin University, Australia*
[◇]*University of Chicago, USA*

## A  Artifact Appendix

### A.1  Abstract

This artifact appendix focuses on two main claims (please refer to "Contributions" in Section 1 and "Main takeaways" at the end of Section 5) in our paper:

- **`CAMP` and *policy imitation* enhance the certified robustness of DRL agents:** We compare `CAMP` with two other baseline training methods, namely Gaussian and NoisyNet. `CAMP` outperforms these baselines in the trade-off between certified agent utility (*i.e.,* "Certified Expected Return") and certified robustness (*i.e.,* "Certified Radius"). In particular, significant performance gains are observed in both classic control and autonomous driving environments.
- **Agents trained by `CAMP` and *policy imitation* are more robust against empirical attacks:** `CAMP` effectively mitigates adversarial perturbations in observations, maintaining a higher average return across diverse attack scenarios. The `CAMP` agents exhibit superior robustness compared to Gaussian agents, with particularly significant improvements observed in classic control and autonomous driving environments.

We have provided the source code and scripts in this artifact to reproduce the results that support these two claims.

### A.2  Description & Requirements

The artifact enables the demonstration and reproduction of results for both `CAMP`, Gaussian, and NoisyNet agents in the Cartpole environment. Since the original code and scripts were developed and tested on our internal high-performance computer (HPC) cluster (CSIRO Virga), which restricts access from external users, we have done our best to ensure the results can be reproduced on non-HPC devices. We provided a list of requirements for running our code on a customer PC to ensure that evaluators can conduct the evaluation.

#### A.2.1  Security, privacy, and ethical concerns

To the best of our knowledge, there are no security, privacy, or ethical concerns associated with our artifact. Our code does not collect data or fingerprints from evaluators. Moreover, these concerns are not applicable if the evaluator uses their own hardware and software platform for evaluation.

#### A.2.2  How to access

The code is available from our GitHub repository (https://github.com/NeuralSec/camp-robust-rl). Evaluators can simply clone the repository to their local PC and set up a virtual environment as instructed. Since training the DQN can be highly unstable, the quality of agents trained with different GPUs or in different system environments may vary. To ensure reproducibility, we set up containers with NVIDIA RTX-4090 on `runpod.io`, allowing evaluators to log in anonymously and run the corresponding experiments.

#### A.2.3  Hardware dependencies

We recommend that evaluators use a PC equipped with a CUDA-capable GPU. While our code and scripts were developed and tested on an internal HPC cluster with NVIDIA H100 GPUs, the code is not memory-intensive. Consumer-grade GPUs supporting CUDA should handle execution without issues. For reference, training with `CAMP` and policy imitation on the Cartpole environment typically consumes less than 1GB VRAM.

#### A.2.4  Software dependencies

The code should be executable on various versions of Linux operating systems. A `requirements.txt` file is provided

Table 1: Environment for Experiments

| Software | Detail |
|---|---|
| Operating system | Linux (SUSE Linux Enterprise Server SLE15) |
| Python version | 3.11.4 |
| CUDA version | 12.4 |
| cuDNN version | 9.1.0 |
| Major libraries | Pytorch 2.5.1<br>Gymnasium 0.29.1<br>Highway-env 1.9.1<br>ale-py 0.8.1<br>autoattack 0.1<br>SciPy 1.11.0<br>Numpy 1.26.4<br>statsmodels 0.14.2<br>matplotlib 3.8.4<br>tensorboard 2.16.2<br>tqdm 4.66.4<br>mlconfig 0.1.8 |

for quick setup of the environment. This file includes a comprehensive list of dependencies, rather than just the minimal required ones, to facilitate the recreation of the virtual environment used during our experiments. Table 1 lists the system environment and major libraries used during our code testing.

### A.2.5 Benchmarks

We mainly produce results based on the Gymnasium [3] Cartpole environment in this artifact. Specifically, this artifact uses `"Cartpole-v0"` from Gymnasium. The Q-network of the agent is a multi-layer perception (MLP) whose architecture is described in Table 4 of our paper.

## A.3 Set-up

- Clone the repository from GitHub to your local PC or the provided container.
- A `readme.md` file included in the repository describes the commands for environment setup.
- The experiment will create new folders within the current directory to store its results. Please ensure that the directory where you run the experiment has the necessary write permissions to allow the creation of these folders.
- We provide four shell scripts to help you reproduce the results.

### A.3.1 Installation

The installation steps are described in the `readme.md` file. First, navigate to the root directory of the cloned repository and create a virtual environment using `venv`. Next, update `pip` to the latest version and install dependencies from `requirements.txt`. After installation, you can type `deactivate` to exit the virtual environment. To this end, the setup is finished and the code should be ready for the following evaluation. We have observed that when installing dependencies on some Ubuntu distributions (tested on WSL2 Ubuntu 24.04.1 LTS), running `pip install -r requirements.txt --no-cache-dir` may result in errors like segmentation fault due to conflicting dependency versions. If this occurs, rerunning the same command immediately after the error often resolves the issue.

### A.3.2 Basic Test

In the root directory of the cloned repository, execute `bash ae_step1.sh`. If the Python version and GPU information are displayed without errors, the setup is successful. To monitor training with TensorBoard, open a new terminal, navigate to the root directory of the repository with the virtual environment activated, and run `tensorboard --logdir="exp"`. This will allow you to visualize the training progress.

## A.4 Evaluation workflow

This artifact includes four shell scripts located in the root directory, which generate results for two major experiments designed to verify our two claims. In this evaluation, we will focus on reproducing results in the Cartpole environment under the setting of $\sigma = 0.2$.

- `ae_step1.sh` trains agents using CAMP, Gaussian, and NoisyNet on `"Cartpole-v0"` with single-frame observations (*i.e.,* "Cartpole-1" in the paper) and $\sigma = 0.2$.
- `ae_step2.sh` tests the three agents trained by `ae_step1.sh` and certifies their robustness. The results are plotted as a `.pdf` figure and saved in the directory `./cert_exp/comparisons/`.
- `ae_step3.sh` attacks the CAMP and Gaussian agents trained by `ae_step1.sh` using projected gradient (PGD) attack [2] and plots the result in a `.pdf` figure in the `./attacks/` folder.
- `ae_step4.sh` attacks the CAMP and Gaussian agents trained by `ae_step1.sh` using Auto PGD (APGD) attack [1] and plots the result in a `.pdf` figure in the `./attacks/` folder.

The results from `ae_step1.sh` and `ae_step2.sh` support our first claim, while those from `ae_step3.sh` and `ae_step4.sh` support our second claim.

### A.4.1 Major Claims

(C1): *CAMP and policy imitation enhance the certified robustness of DRL agents. This is demonstrated by the experimental results in Figures 2 and 8 of the paper.*

(C2): *CAMP and policy imitation effectively improves the robustness of agents against empirical attacks perturbing observations. This is demonstrated by the experimental results in Figures 3 and 4 of the paper.*

### A.4.2 Experiments

**(E1):** *[Train agents with different methods on Cartpole-1 and certify the robustness of trained agents] [1 human-minute + 2.5 compute-hour]: Train agents with different methods and certify their robustness to support (C1).*
**How to:**
**Preparation:** *Complete installation described in Section A.3 first.*
**Execution:** *1) In the root directory of the cloned repository, run `bash ae_step1.sh`. The script trains `CAMP`, Gaussian, and NoisyNet agents in the Cartpole-1 environment with the noise scale $\sigma = 0.2$. The trained agents are saved under the `./exp/` folder. 2) When finished, run `bash ae_step2.sh`. This script loads the trained agents and tests each agent in Cartpole-1 for $10{,}000$ independent episodes. The test rewards will be saved under `./eval_exp/`. Then the test rewards will be loaded and used for robustness certification. The certification results will be plotted in `./exp_cert/cartpole_simple-cert_dqn.pdf`.*
**Results:** *The plotted results should closely resemble those in Figure 2 (sub-figure "Cartpole-1" when $\sigma = 0.2$) of the paper. The certified curve for `CAMP` should be higher than those of the Gaussian and NoisyNet baselines.*

**(E2):** *[Apply PGD and APGD attacks to evaluate empirical robustness.] [1 human-minute + 2 compute-hour]: Compare the robustness of `CAMP` agent and Gaussian agent to support (C2).*
**How to:**
**Preparation:** *Complete the installation described in Section A.3 and (E1) to ensure that trained agents are in place.*
**Execution:** *1) In the root directory of the cloned repository, run `bash ae_step3.sh`. This script attacks previously trained `CAMP` and Gaussian agents using PGD across different perturbation budgets. The attack results will be saved under `./exp/`. The comparison between `CAMP` and Gaussian agents will be plotted in `attacks/cartpole_simple_pgd_threshold=0.pdf`. 2) Run `bash ae_step4.sh`. This script attacks previously trained `CAMP` and Gaussian agents using APGD across different perturbation budgets. The attack results will again be saved under `./exp/`. The comparison between `CAMP` and Gaussian agents will be plotted in `cartpole_simple_apgd_threshold=0.pdf`.*
**Results:** *To reduce time costs for evaluators, we have decreased the number of evaluations from 1,000 to 100 when computing the average return in this artifact evaluation. The expected results should reflect the trends in Figures 3 and 4 (sub-figure "Cartpole-1" when $\sigma = 0.2$). The average return of the `CAMP` agent should be significantly higher than that of the Gaussian agent.*

## A.5 Notes on Reusability

We have included four scripts in our code repository to enable a fuss-free reproduction of the training-testing-certification pipeline and empirical robustness evaluation tasks in "Cartpole-1" under a specific hyper-parameter setting ($\sigma = 0.2$). This configuration represents one of the major improvement cases (*i.e.,* "Cartpole" and "Highway"), and other configurations can be evaluated by simply setting the `env_id` and `env_sigma` arguments in the scripts. Beyond these scripts, detailed instructions for running experiments in different environments with various hyper-parameter settings are available in the `readme.md` file. This also enables others to adapt `CAMP` and policy imitation to agents and environments beyond those tested in our paper. For comprehensive information and access to the source code, please visit our GitHub repository: https://github.com/NeuralSec/camp-robust-rl.

## A.6 Version

Based on the LaTeX template for Artifact Evaluation V20231005. Submission, reviewing and badging methodology followed for the evaluation of this artifact can be found at https://secartifacts.github.io/usenixsec2025/.



\end{document}